\documentclass{article}

\usepackage[final]{nips_2018}

\usepackage[utf8]{inputenc} 
\usepackage[T1]{fontenc}    
\usepackage{hyperref}       
\usepackage{url}            
\usepackage{booktabs}       
\usepackage{amsfonts}       
\usepackage{nicefrac}       
\usepackage{microtype}      
\usepackage{graphicx}
\usepackage{epstopdf}
\usepackage{tabularx}
\usepackage{afterpage}
\usepackage{amsmath,amssymb}            
\usepackage{rotating}  
\usepackage{fancyhdr}  
\usepackage{enumerate}
\usepackage{caption} 
\usepackage{dsfont}
\usepackage{bm}
\usepackage{mathtools}
\usepackage{wrapfig}
\usepackage{lipsum}
\usepackage{etoolbox}
\usepackage{algpseudocode}
\usepackage{algorithm}
\usepackage{verbatim}
\usepackage{capt-of}
\usepackage{emptypage}
\usepackage{multicol}
\usepackage{subcaption}
\usepackage{adjustbox}
\usepackage{wrapfig}
\usepackage{caption}
\usepackage{lscape}
\usepackage{tabularx}
\usepackage[table]{xcolor}
\usepackage{arydshln}
\usepackage{enumitem}

\usepackage[colorinlistoftodos, textsize=tiny]{todonotes}
\definecolor{citrine}{rgb}{0.89, 0.82, 0.04}
\definecolor{blued}{RGB}{70,197,221}

\AtBeginEnvironment{proof}{\small}
\allowdisplaybreaks[4]

\usepackage{xspace}                                                             
\DeclareRobustCommand{\eg}{e.g.,\@\xspace}                                      
\DeclareRobustCommand{\ie}{i.e.,\@\xspace}                                      
\DeclareRobustCommand{\wrt}{w.r.t.\@\xspace}                                    

\DeclareRobustCommand{\quotes}[1]{``#1''}

\usepackage{bm}
\newcommand{\mathbr}[1]{\bm{\mathbf{#1}}}
\newcommand{\ess}{\mathrm{ESS}}

\usepackage{amsthm}  
\usepackage{thmtools}
\usepackage{thm-restate}

\usepackage[english]{babel}

\DeclareMathOperator*{\ev}{\mathbb{E}}
\DeclareMathOperator*{\Var}{\mathbb{V}\mathrm{ar}}

\DeclareMathOperator*{\esssup}{ess\,sup}

\newcommand{\Renyi}{R\'{e}nyi }
\newcommand{\de}{\,\mathrm{d}}
\newcommand{\vtheta}{\mathbr{\theta}}
\newcommand{\vrho}{\mathbr{\rho}}

\newcommand{\hyscoreprime}[1][\vtheta]{\nabla_{\vrho'}\log\nu_{\vrho'}}

\title{Policy Optimization via Importance Sampling}

\author{
  Alberto Maria Metelli\\
  Politecnico di Milano, Milan, Italy\\
  \footnotesize\texttt{\href{mailto:albertomaria.metelli@polimi.it}{albertomaria.metelli@polimi.it}} \\
  \And
  Matteo Papini\\
  Politecnico di Milano, Milan, Italy\\
  \footnotesize\texttt{\href{mailto:matteo.papini@polimi.it}{matteo.papini@polimi.it}} \\
  \And
  Francesco Faccio\\
  Politecnico di Milano, Milan, Italy\\ IDSIA, USI-SUPSI, Lugano, Switzerland\\
  \footnotesize\texttt{\href{mailto:francesco.faccio@mail.polimi.it}{francesco.faccio@mail.polimi.it}} \\
  \And
  Marcello Restelli\\
  Politecnico di Milano, Milan, Italy\\
  \footnotesize\texttt{\href{mailto:marcello.restelli@polimi.it}{marcello.restelli@polimi.it}} \\
}

\begin{document}



\maketitle

\begin{abstract}
	Policy optimization is an effective reinforcement learning approach to solve continuous control tasks. Recent achievements have shown that
	alternating online and offline optimization is a successful choice for efficient
	trajectory reuse. However, deciding when to stop optimizing and collect new trajectories is non-trivial, as it requires to account for the variance of the objective function estimate. In this paper, we propose 
	a novel, model-free, policy search algorithm, POIS, applicable in both action-based and parameter-based settings. We first derive a high-confidence bound
	for importance sampling estimation; then we define a surrogate objective function, which is optimized offline whenever a new batch of trajectories is collected. 
	Finally, the algorithm is tested on a selection of continuous control tasks, with both linear and deep policies, and compared with state-of-the-art policy optimization methods.
\end{abstract}

\section{Introduction}
In recent years, policy search methods~\cite{deisenroth2013survey} have proved to be valuable Reinforcement Learning (RL)~\cite{sutton1998reinforcement} approaches thanks to their successful achievements in continuous control tasks~\cite[\eg][]{lillicrap2015continuous, schulman2015trust, schulman2017proximal, schulman2015high}, robotic locomotion~\cite[\eg][]{tedrake2004stochastic, kober2013reinforcement} and partially observable environments~\cite[\eg][]{ng2000pegasus}. These algorithms can be roughly classified into two categories: \emph{action-based} methods~\cite{sutton2000policy, peters2008reinforcement} and \emph{parameter-based} methods~\cite{sehnke2008policy}. The former, usually known as policy gradient (PG) methods, perform a search in a parametric policy space by following the gradient of the utility function estimated by means of a batch of trajectories collected from the environment~\cite{sutton1998reinforcement}. In contrast, in parameter-based methods, the search is carried out directly in the space of parameters by exploiting global optimizers~\cite[\eg][]{rubinstein1999cross, hansen2001completely, stanley2002evolving, szita2006learning} or following a proper gradient direction like in Policy Gradients with Parameter-based Exploration (PGPE)~\cite{sehnke2008policy, wierstra2008natural, sehnke2010parameter}. A major question in policy search methods is: how should we use a batch of trajectories in order to exploit its information in the most efficient way?
On one hand, \emph{on-policy} methods leverage on the batch to perform a single gradient step, after which new trajectories are collected with the updated policy. Online PG methods are likely the most widespread policy search approaches: starting from the traditional algorithms based on stochastic policy gradient~\cite{sutton2000policy}, like REINFORCE~\cite{williams1992simple} and G(PO)MDP~\cite{baxter2001infinite}, moving toward more modern methods, such as Trust Region Policy Optimization (TRPO)~\cite{schulman2015trust}. These methods, however, rarely exploit the available trajectories in an efficient way, since each batch is
thrown away after just one gradient update. 
On the other hand, \emph{off-policy} methods maintain a behavioral policy, used to explore the environment and to collect samples, and a target policy which is optimized. The concept of off-policy learning is rooted in value-based RL~\cite{watkins1992q, peng1994incremental, munos2016safe} and it was first adapted to PG in~\cite{degris2012off}, using an actor-critic architecture. The approach has been extended to Deterministic Policy Gradient (DPG)~\cite{silver2014deterministic}, which allows optimizing deterministic policies while keeping a stochastic policy for exploration. More recently,
an efficient version of DPG coupled with a deep neural network to represent the policy has been proposed, named Deep Deterministic Policy Gradient (DDPG)~\cite{lillicrap2015continuous}. In the parameter-based framework, even though the original formulation~\cite{sehnke2008policy} introduces an online algorithm, an extension has been proposed to efficiently reuse the trajectories in an offline scenario~\cite{zhao2013efficient}. Furthermore, PGPE-like approaches allow overcoming several limitations of classical PG, like the need for a stochastic policy and the high variance of the gradient estimates.\footnote{Other solutions to these problems have been proposed in the action-based literature, like the aforementioned DPG algorithm, the gradient baselines~\cite{peters2008reinforcement} and the actor-critic architectures~\cite{konda2000actor}.} 

While on-policy algorithms are, by nature, \emph{online}, as they need to be fed with fresh samples whenever the policy is updated, off-policy methods can take advantage of mixing online and \emph{offline} optimization. This can be done by alternately sampling trajectories and performing optimization epochs with the collected data. A prime example of this alternating procedure is Proximal Policy Optimization (PPO)~\cite{schulman2017proximal}, that has displayed remarkable performance on continuous control tasks. Off-line optimization, however, introduces further sources of approximation, as the gradient \wrt the target policy needs to be estimated (off-policy) with samples collected with a behavioral policy. A common choice is to adopt an \emph{importance sampling} (IS)~\cite{mcbook, hesterberg1988advances} estimator in which each sample is reweighted proportionally to the likelihood of being generated by the target policy. However, directly optimizing this utility function is impractical since it displays a wide variance most of the times~\cite{mcbook}. Intuitively, the variance increases proportionally to the distance between the behavioral and the target policy; thus, the estimate is reliable as long as the two policies are close enough. Preventing uncontrolled updates in the space of policy parameters is at the core of the natural gradient approaches~\cite{amari1998natural} applied effectively both on PG methods~\cite{kakade2002approximately, peters2008natural, wierstra2008natural} and on PGPE methods~\cite{miyamae2010natural}. More recently, this idea has been captured (albeit indirectly) by TRPO, which optimizes via (approximate) natural gradient a surrogate objective function, derived from safe RL~\cite{kakade2002approximately, pirotta2013safe}, subject to a constraint on the Kullback-Leibler divergence between the behavioral and target policy.\footnote{Note that this regularization term appears in the performance improvement bound, which contains exact quantities only. Thus, it does not really account for the uncertainty derived from the importance sampling.} Similarly, PPO performs a truncation of the importance weights to discourage the optimization process from going too far. Although TRPO and PPO, together with DDPG, represent the state-of-the-art policy optimization methods in RL for continuous control, they do not explicitly encode in their objective function the uncertainty injected by the importance sampling procedure. A more theoretically grounded analysis has been provided for policy selection~\cite{doroudi2017importance}, model-free~\cite{thomas2015high} and model-based~\cite{thomas2016data} policy evaluation (also accounting for samples collected with multiple behavioral policies), and combined with options~\cite{guo2017using}. Subsequently, in~\cite{thomas2015high2} these methods have been extended for policy improvement, deriving a suitable concentration inequality for the case of truncated importance weights. Unfortunately, these methods are hardly scalable to complex control tasks. A more detailed review of the state-of-the-art policy optimization algorithms is reported in Appendix~\ref{apx:relatedWorks}.

In this paper, we propose a novel, model-free, actor-only, policy optimization algorithm, named \emph{Policy Optimization via Importance Sampling} (POIS) that mixes online and offline optimization to efficiently exploit the information contained in the collected trajectories. POIS explicitly accounts for the uncertainty introduced by the importance sampling by optimizing a surrogate objective function. The latter captures the trade-off
between the estimated performance improvement and the variance injected by the importance sampling. 
The main contributions of this paper are
theoretical, algorithmic and experimental. After revising some notions about importance sampling (Section~\ref{sec:is}), we propose a concentration inequality, of independent interest, for high-confidence \quotes{off-distribution} optimization of objective functions estimated via importance sampling (Section~\ref{sec:optimization}). Then we
show how this bound can be customized into a surrogate objective function in order to either search in the space of policies (Action-based POIS) or to search in the space of parameters (Parameter-bases POIS). The resulting algorithm (in both the action-based and the parameter-based flavor) collects, at each iteration, a set of trajectories. These are used to perform offline optimization of the surrogate objective via gradient ascent (Section~\ref{sec:policyOptimization}), after which a new batch of trajectories is collected using the optimized policy. Finally, we provide an experimental evaluation with both linear policies and deep neural policies to illustrate the advantages and limitations of our approach compared to state-of-the-art algorithms (Section~\ref{sec:experimental}) on classical control tasks~\cite{duan2016benchmarking, todorov2012mujoco}. The proofs for all Theorems and Lemmas are reported in Appendix~\ref{apx:proofs}. The implementation of POIS can be found at \url{https://github.com/T3p/pois}.

\section{Preliminaries}
\label{sec:preliminaries}
A discrete-time Markov Decision Process (MDP)~\cite{puterman2014markov} is defined as
a tuple $\mathcal{M}=(\mathcal{S},\mathcal{A},P,R,\gamma,D)$ where
$\mathcal{S}$ is the state space, $\mathcal{A}$ is the action space,
$P(\cdot|s,a)$ is a Markovian transition model that assigns for each state-action pair $(s,a)$ the probability of reaching the next state $s'$, $\gamma\in{[0,1]}$ is the discount factor, $R(s,a)\in [ -R_{\max}, R_{\max}]$ assigns the expected reward for performing action $a$ in state $s$ and $D$ is the distribution of the initial state. The behavior of an agent is described by a  policy $\pi(\cdot|s)$ that assigns for each state $s$ the probability of performing action $a$. A trajectory $\tau\in \mathcal{T}$ is a sequence of state-action pairs $\tau = (s_{\tau,0},a_{\tau,0},\dots,s_{\tau,H-1},a_{\tau,H-1},s_{\tau,H})$, where $H$ is the actual trajectory horizon. The performance of an agent is evaluated in terms of the \emph{expected return}, \ie the expected discounted sum of the rewards collected along the trajectory: $\ev_{\tau} \left[ R(\tau) \right]$, where $R(\tau) = \sum_{t=0}^{H-1} \gamma^t R(s_{\tau,t}, a_{\tau,t})$ is the trajectory return.

We focus our attention to the case in which the policy belongs to a parametric policy space $\Pi_\Theta = \{ \pi_{\mathbr{\theta}} : \mathbr{\theta} \in \Theta \subseteq \mathbb{R}^p \}$. In parameter-based approaches, the agent is equipped with a \emph{hyperpolicy} $\nu$ used to sample the policy parameters at the beginning of each episode. The hyperpolicy belongs itself to a parametric hyperpolicy space $\mathcal{N}_{\mathcal{P}} = \{ \nu_{\mathbr{\rho}} : \mathbr{\rho} \in \mathcal{P} \subseteq \mathbb{R}^r \}$. 
The expected return can be expressed, in the parameter-based case, as a double expectation: one over the policy parameter space $\Theta$ and one over the trajectory space $\mathcal{T}$:
\begin{equation}
\label{eq:expectedReturn}
	J_D(\mathbr{\rho}) = \int_{{\Theta}}  \int_{\mathcal{T}} \nu_{\mathbr{\rho}} (\mathbr{\theta}) p(\tau | \mathbr{\theta}) R(\tau) \de \tau \de \mathbr{\theta},
\end{equation}
where $p(\tau | \mathbr{\theta})=D(s_0)\prod_{t=0}^{H-1} \pi_{\mathbr{\theta}}(a_t|s_t) P(s_{t+1}|s_t,a_t)$ is the trajectory density function. The goal of a parameter-based learning agent is to determine the hyperparameters $\vrho^*$ so as to maximize $J_D(\vrho)$. If $\nu_{\mathbr{\rho}}$ is stochastic and differentiable, the hyperparameters can be learned according to the gradient ascent update: $\mathbr{\rho}' = \mathbr{\rho} + \alpha \nabla_{\mathbr{\rho}} J_D(\mathbr{\rho})$, where $\alpha > 0$ is the step size and $\nabla_{\mathbr{\rho}} J_D(\mathbr{\rho}) = \int_{\Theta}  \int_{\mathcal{T}} \nu_{\mathbr{\rho}}(\mathbr{\theta}) p(\tau | \mathbr{\theta}) \nabla_{\mathbr{\rho}} \log \nu_{\mathbr{\rho}} (\mathbr{\theta})  R(\tau) \de \tau \de \mathbr{\theta}$. Since the stochasticity of the hyperpolicy is a sufficient source of exploration, deterministic action policies of the kind $\pi_{\vtheta}(a|s) = \delta(a - u_{\vtheta}(s))$ are typically considered, where $\delta$ is the Dirac delta function and $u_{\vtheta}$ is a deterministic mapping from $\mathcal{S}$ to $\mathcal{A}$. In the action-based case, on the contrary, the hyperpolicy $\nu_{\mathbr{\rho}}$ is a deterministic distribution $\nu_{\mathbr{\rho}}(\mathbr{\theta}) = \delta (\mathbr{\theta} - g({\mathbr{\rho}}) )$, where $g({\mathbr{\rho}})$ is a deterministic mapping from $\mathcal{P}$ to $\Theta$. For this reason, the dependence on $\mathbr{\rho}$ is typically not represented and the expected return expression simplifies into a single expectation over the trajectory space $\mathcal{T}$:
\begin{equation}
	J_D(\mathbr{\theta}) = \int_{\mathcal{T}} p(\tau | \mathbr{\theta}) R(\tau) \de \tau .
\end{equation}
An action-based learning agent aims to find the policy parameters $\mathbr{\theta}^*$ that maximize $J_D(\mathbr{\theta})$. In this case, we need to enforce exploration by means of the stochasticity of $\pi_{\mathbr{\theta}}$. For stochastic and differentiable policies, learning can be performed via gradient ascent: $\mathbr{\theta}' = \mathbr{\theta} + \alpha \nabla_{\mathbr{\theta}} J_D(\mathbr{\theta})$, where $\nabla_{\mathbr{\theta}} J_D(\mathbr{\theta}) = \int_{\mathcal{T}} p(\tau | \mathbr{\theta}) \nabla_{\mathbr{\theta}} \log p(\tau | \mathbr{\theta}) R(\tau) \de \tau$.

\section{Evaluation via Importance Sampling}
\label{sec:is}
In off-policy evaluation~\cite{thomas2015high, thomas2016data}, we aim to estimate the performance of a target policy $\pi_T$ (or hyperpolicy $\nu_T$) given samples collected with a behavioral policy $\pi_B$ (or hyperpolicy $\nu_B$). More generally, we face the problem of estimating the expected value of a deterministic bounded function $f$ ($\|f \|_{\infty} < +\infty$) of random variable $x$ taking values in $\mathcal{X}$ under a target distribution $P$, after having collected samples from a behavioral distribution $Q$. The \emph{importance sampling} estimator (IS)~\cite{cochran2007sampling, mcbook} corrects the distribution with the \emph{importance weights} (or Radon–Nikodym derivative or likelihood ratio) $w_{P/Q}(x) = p(x)/q(x)$:
\begin{equation}
\label{eq:is_estimator}
	\widehat{\mu}_{P/Q} = \frac{1}{N} \sum_{i=1}^N \frac{p(x_i)}{q(x_i)} f(x_i) = \frac{1}{N} \sum_{i=1}^N w_{P/Q}(x_i) f(x_i),
\end{equation}
where $\mathbr{x} = (x_1,x_2,\dots,x_N)^T$ is sampled from $Q$ and we assume $q(x) > 0$ whenever $f(x)p(x) \neq 0$. This estimator is unbiased ($\ev_{\mathbr{x}\sim Q}[\widehat{\mu}_{P/Q}] = \ev_{x\sim P}[f(x)]$) but it may exhibit an undesirable behavior due to the variability of the importance weights, showing, in some cases, infinite variance. Intuitively, the magnitude of the importance weights provides an indication of how much the probability measures $P$ and $Q$ are dissimilar. This notion can be formalized by the \Renyi divergence~\cite{renyi1961measures, van2014renyi}, an information-theoretic dissimilarity index
between probability measures.

\paragraph{\Renyi divergence}  Let $P$ and $Q$ be two probability measures on a measurable space $(\mathcal{X}, \mathcal{F})$ such that $P \ll Q$ ($P$ is absolutely continuous \wrt $Q$) and $Q$ is $\sigma$-finite. Let $P$ and $Q$ admit $p$ and $q$ as Lebesgue probability density functions (p.d.f.), respectively. The $\alpha$-\Renyi divergence is defined as:
\begin{equation}
\label{eq:renyiDiv}
	D_{\alpha} (P \| Q) = \frac{1}{\alpha-1} \log \int_\mathcal{X} \left( \frac{\de P}{\de Q} \right)^{\alpha} \de Q  = \frac{1}{\alpha-1} \log \int_\mathcal{X} q(x) \left( \frac{ p(x)}{ q(x) } \right)^{\alpha}  \de x,
\end{equation} 
where $\de P/\de Q $ is the Radon–Nikodym derivative of $P$ \wrt $Q$ and $\alpha \in [0, \infty]$. Some remarkable cases are: $\alpha = 1$ when $D_{1}(P \| Q) = D_{\mathrm{KL}}(P \| Q)$ and $\alpha=\infty$ yielding $D_{\infty}(P \| Q) = \log \esssup_\mathcal{X}  \de P / \de Q$.
Importing the notation from~\cite{cortes2010learning}, we indicate the exponentiated $\alpha$-\Renyi divergence as $d_{\alpha} (P \| Q) = \exp \left( D_{\alpha} (P \| Q) \right)$.
With little abuse of notation, we will replace $D_{\alpha}(P\| Q)$ with $D_{\alpha}(p\| q)$ whenever possible within the context.

The \Renyi divergence provides a convenient expression for the moments of the importance weights: $\ev_{x \sim Q} \left[ w_{P/Q}(x)^\alpha \right] = d_{\alpha} (P \| Q)$. Moreover, $\Var_{x \sim Q} \left[ w_{P/Q}(x) \right] = d_2 (P \| Q) - 1$ and $\esssup_{x \sim Q}  w_{P/Q}(x) = d_\infty (P \| Q)$~\cite{cortes2010learning}. To mitigate the variance problem of the IS estimator, we can resort to the \emph{self-normalized importance sampling} estimator (SN)~\cite{cochran2007sampling}:
\begin{equation}
\label{eq:sn_estimator}
	\widetilde{\mu}_{P/Q} = \frac{\sum_{i=1}^N w_{P/Q}(x_i) f(x_i)}{\sum_{i=1}^N w_{P/Q}(x_i)}  =  \sum_{i=1}^N \widetilde{w}_{P/Q}(x_i) f(x_i),
\end{equation}
where $ \widetilde{w}_{P/Q}(x) =w_{P/Q}(x) /\sum_{i=1}^N w_{P/Q}(x_i)$ is the self-normalized importance weight. Differently from $\widehat{\mu}_{P/Q}$, $\widetilde{\mu}_{P/Q}$ is biased but consistent~\cite{mcbook} and it typically displays a more desirable behavior because of its smaller variance.\footnote{Note that $\left| \widetilde{\mu}_{P/Q} \right| \le \|f\|_{\infty}$. Therefore, its variance is always finite.} Given the realization $x_1,x_2,\dots,x_N$ we can interpret the SN estimator as the expected value of $f$ under an approximation of the distribution $P$ made by $N$ deltas, \ie $\widetilde{p}(x) = \sum_{i=1}^N \widetilde{w}_{P/Q}(x) \delta(x-x_i)$. The problem of assessing the quality of the SN estimator has been extensively studied by the simulation community, producing several diagnostic indexes to indicate when the weights might display
problematic behavior~\cite{mcbook}. The \emph{effective sample size} ($\ess$) was introduced in~\cite{kong1992note} as the number of samples drawn from $P$ so that the variance of the Monte Carlo estimator $\widetilde{\mu}_{P/P}$ is approximately equal to the variance of the SN estimator $\widetilde{\mu}_{P/Q}$ computed with $N$ samples. Here we report the original definition and its most common estimate:
\begin{equation} \label{eq:ess}
	\ess (P \| Q) = \frac{N}{\Var_{x\sim Q} \left[ w_{P/Q} (x) \right] + 1} = \frac{N}{d_2 (P \| Q)}, \quad \widehat{\ess } (P \| Q)  = \frac{1}{\sum_{i=1}^N \widetilde{w}_{P/Q}(x_i)^2}.
\end{equation}
The $\ess$ has an interesting interpretation: if $d_2(P \| Q)=1$, \ie $P=Q$ almost everywhere, then $\ess = N$ since we are performing Monte Carlo estimation. Otherwise, the $\ess$ decreases as the dissimilarity between the two distributions increases. In the literature, other $\ess$-like diagnostics have been proposed that also account for the nature of $f$~\cite{martino2017effective}.

\section{Optimization via Importance Sampling}
\label{sec:optimization}
The off-policy optimization problem~\cite{thomas2015high2} can be formulated as finding the best target policy $\pi_T$ (or hyperpolicy $\nu_T$), \ie the one maximizing the expected return, having access to a set of samples collected with a behavioral policy $\pi_B$ (or hyperpolicy $\nu_B$). In a more abstract sense, we aim to determine the target distribution $P$ that maximizes $\ev_{x\sim P}[f(x)]$ having samples 
collected from the fixed behavioral distribution $Q$. In this section, we analyze the problem of defining a proper objective function for this
purpose. Directly optimizing the estimator $\widehat{\mu}_{P/Q}$ or $\widetilde{\mu}_{P/Q}$ is, in most of the cases, unsuccessful. With enough freedom in choosing $P$, the optimal solution would assign as much probability mass as possible to the maximum value among $f(x_i)$. Clearly, in this scenario, the estimator is unreliable and displays a large variance. For this reason, we adopt a risk-averse approach and we decide to optimize a statistical \emph{lower bound} of the expected value $\ev_{x\sim P}[f(x)]$ that holds with high confidence.
We start by analyzing the behavior of the IS estimator and we provide the following result that bounds the variance of $\widehat{\mu}_{P/Q}$ in terms of the Renyi divergence.
\begin{restatable}[]{lemma}{boundVariance}
Let $P$ and $Q$ be two probability measures on the measurable space $\left(\mathcal{X}, \mathcal{F} \right)$ such that $P \ll Q$. Let $\mathbr{x} = (x_1,x_2,\dots,x_N)^T$ i.i.d. random variables sampled from $Q$ and $f: \mathcal{X} \rightarrow \mathbb{R}$ be a bounded function ($\| f \|_{\infty}<+\infty$). Then, for any $N>0$, the variance of the IS estimator $\widehat{\mu}_{P/Q}$ can be upper bounded as:
\begin{equation}
\label{eq:varianceBound}
	\Var_{\mathbr{x} \sim Q}  \left[ \widehat{\mu}_{P/Q} \right] \le \frac{1}{N} \|f\|_{\infty}^2 d_2 \left( P \| Q \right).
\end{equation}
\end{restatable}
When $P = Q$ almost everywhere, we get $\Var_{\mathbr{x} \sim Q}  \left[ \widehat{\mu}_{Q/Q} \right] \le \frac{1}{N} \|f\|_{\infty}^2$, a well-known bound on the variance of a Monte Carlo estimator. Recalling the definition of ESS~\eqref{eq:ess} we can rewrite the previous bound as: $\Var_{\mathbr{x} \sim Q}  \left[ \widehat{\mu}_{P/Q} \right] \le \frac{\|f\|_{\infty}^2 }{\ess (P \| Q)}$, \ie the variance scales with ESS instead of $N$.
While $\widehat{\mu}_{P/Q}$ can have unbounded variance even if $f$ is bounded, the SN estimator $\widetilde{\mu}_{P/Q}$ is always bounded by $\|f\|_{\infty}$ and therefore it always has a finite variance. Since the normalization term makes all the samples $\widetilde{w}_{P/Q}(x_i) f(x_i)$ interdependent, an exact analysis of its bias and variance is more challenging. Several works adopted approximate methods to provide an expression for the variance~\cite{hesterberg1988advances}. We propose an analysis of bias and variance of the SN estimator in Appendix~\ref{apx:SN}.

\subsection{Concentration Inequality}
Finding a suitable concentration inequality for off-policy learning was studied in~\cite{thomas2015high} for offline policy evaluation and subsequently in~\cite{thomas2015high2} for optimization. 
On one hand, fully empirical concentration inequalities, like Student-T, besides the asymptotic approximation, are not suitable in this case since the empirical variance needs to be estimated with importance sampling as well injecting further uncertainty~\cite{mcbook}. On the other hand, several distribution-free inequalities like Hoeffding require knowing the maximum of the estimator, which might not exist ($d_{\infty}(P\| Q) = \infty$) for the IS estimator. Constraining $d_{\infty}(P\| Q)$ to be finite often introduces unacceptable limitations. For instance, in the case of univariate Gaussian distributions, it prevents a step that selects a target variance larger than the behavioral one from being performed (see Appendix~\ref{apx:IS}).\footnote{Although the variance
tends to be reduced in the learning process, there might be cases in which it needs to be increased (\eg suppose we start
with a behavioral policy with small variance, it might be beneficial increasing the variance to enforce exploration).} Even Bernstein inequalities~\cite{bercu2015concentration}, are hardly applicable since, for instance, in the case of univariate Gaussian distributions, the importance weights display a fat tail behavior (see Appendix~\ref{apx:IS}). We believe that a reasonable trade-off is to require the variance of the importance weights to be finite, that is equivalent to require $d_2(P\| Q) < \infty$, \ie $\sigma_P < 2\sigma_Q$ for univariate Gaussians. For this reason, we resort to Chebyshev-like inequalities and we propose the following concentration bound derived from Cantelli's inequality and customized for the IS estimator.
\begin{restatable}[]{thr}{bound}
\label{thr:bound}
	Let $P$ and $Q$ be two probability measures on the measurable space $\left(\mathcal{X}, \mathcal{F} \right)$ such that $P \ll Q$ and $d_2(P \| Q) < +\infty$. Let $x_1,x_2,\dots,x_N$ be i.i.d. random variables sampled from $Q$, and $f: \mathcal{X} \rightarrow \mathbb{R}$ be a bounded function ($\| f \|_{\infty}<+\infty$). Then, for any $0< \delta \le 1$ and $N>0$ with probability at least $1-\delta$ it holds that:
	\begin{equation}
	\label{eq:bound}
		\ev_{x \sim P} \left[ f(x) \right] \ge \frac{1}{N} \sum_{i=1}^N w_{P/Q}(x_i) f(x_i) - \| f \|_{\infty} \sqrt{\frac{(1-\delta)d_2(P \| Q)}{\delta N}}.
	\end{equation}
\end{restatable}
The bound highlights the interesting trade-off between the estimated performance and the uncertainty introduced by changing the distribution. The latter enters in the bound as the 2-\Renyi divergence between the target distribution $P$ and the behavioral distribution $Q$. Intuitively, we should trust  the estimator $\widehat{\mu}_{P/Q}$ as long as $P$ is not too far from $Q$. 
For the SN estimator, accounting for the bias, we are able to obtain a bound (reported in Appendix~\ref{apx:SN}), with a similar dependence on $P$ as in Theorem~\ref{thr:bound}, albeit with different constants. 
Renaming all constants involved in the bound of Theorem~\ref{thr:bound} as $\lambda =  \| f \|_{\infty} \sqrt{(1-\delta)/\delta}$, we get a surrogate objective function. The optimization can be carried out in different ways. The following section shows why using the natural gradient could be a successful choice in case $P$ and $Q$ can be expressed as parametric differentiable distributions.
\subsection{Importance Sampling and Natural Gradient}
We can look at a parametric distribution $P_{\mathbr{\omega}}$, having $p_{\mathbr{\omega}}$ as a density function, as a point on a probability manifold with coordinates $\mathbr{\omega} \in \Omega$. 
If $p_{\mathbr{\omega}}$ is differentiable, the Fisher Information Matrix (FIM)~\cite{rao1992information, amari2012differential} is defined as: $\mathcal{F}(\mathbr{\omega}) = \int_{\mathcal{X}}p_{\mathbr{\omega}}(x) \nabla_{\mathbr{\omega}}\log p_{\mathbr{\omega}}(x)\nabla_{\mathbr{\omega}}\log p_{\mathbr{\omega}}(x)^T\de x$. This matrix is, up to a scale, an invariant metric \cite{amari1998natural} on parameter space $\Omega$, \ie $\kappa{(\mathbr{\omega}' - \mathbr{\omega})^T\mathcal{F}(\mathbr{\omega})(\mathbr{\omega}' - \mathbr{\omega})}$ is independent on the specific parameterization and provides a second order approximation of the distance between $p_{\mathbr{\omega}}$ and $p_{\mathbr{\omega}'}$ on the probability manifold up to a scale factor $\kappa\in\mathbb{R}$. 
Given a loss function $\mathcal{L}(\mathbr{\omega})$, we define the natural gradient~\cite{amari1998natural,kakade2002natural} as $\widetilde{\nabla}_{\mathbr{\omega}}\mathcal{L}(\mathbr{\omega})=\mathcal{F}^{-1}(\mathbr{\omega}) \nabla_{\mathbr{\omega}}\mathcal{L}(\mathbr{\omega})$, which represents the steepest ascent direction in the probability manifold. Thanks to the invariance property, there is a tight connection between the geometry induced by the \Renyi divergence and the Fisher information metric.
\begin{restatable}[]{thr}{riemannMetric}
	Let $p_{\mathbr{\omega}}$ be a p.d.f. differentiable \wrt $\mathbr{\omega} \in \Omega$. Then, it holds that, for the \Renyi divergence: $D_{\alpha} (p_{\mathbr{\omega}'} \| p_{\mathbr{\omega}}) = \frac{\alpha}{2} \left(\mathbr{\omega}' - \mathbr{\omega} \right)^T \mathcal{F}(\mathbr{\omega}) \left(\mathbr{\omega}' - \mathbr{\omega} \right) + o(\| \mathbr{\omega}' - \mathbr{\omega} \|_2^2)$, and for the exponentiated \Renyi divergence: $d_{\alpha} (p_{\mathbr{\omega}'} \| p_{\mathbr{\omega}}) = 1 + \frac{\alpha}{2} \left(\mathbr{\omega}' - \mathbr{\omega} \right)^T \mathcal{F}(\mathbr{\omega}) \left(\mathbr{\omega}' - \mathbr{\omega} \right) +o(\| \mathbr{\omega}' - \mathbr{\omega} \|_2^2)$. 
\end{restatable}
This result provides an approximate expression for the variance of the importance weights, as $\Var_{x \sim p_{\mathbr{\omega}}} \left[ w_{\mathbr{\omega}'/\mathbr{\omega}}(x) \right] = d_2 (p_{\mathbr{\omega}'} \| p_{\mathbr{\omega}}) - 1 \simeq \frac{\alpha}{2} \left(\mathbr{\omega}' - \mathbr{\omega} \right)^T \mathcal{F}(\mathbr{\omega}) \left(\mathbr{\omega}' - \mathbr{\omega} \right)$. It also justifies the use of natural gradients in off-distribution optimization, since a step in natural gradient direction has a controllable effect on the variance of the importance weights. 

\begin{figure}[!t]
\begin{minipage}[t]{0.48\textwidth}
  \vspace{0pt}  
  \begin{algorithm}[H]
  \small
  \caption{Action-based POIS}
  \label{alg:cbPOIS}
    \begin{algorithmic}
\State Initialize $\mathbr{\theta}^{0}_{0}$ arbitrarily
\For{$j = 0,1,2,..., \text{ until convergence}$}
	\State Collect $N$ trajectories with $\pi_{\mathbr{\theta}^{j}_{0}}$
	\For{$k = 0,1,2,..., \text{ until convergence}$}
		\State Compute ${\mathcal{G}}(\mathbr{\theta}^{j}_{k})$, $\nabla_{\mathbr{\theta}_{k}^{j}}{\mathcal{L}}(\mathbr{\theta}^{j}_{k} / \mathbr{\theta}^{j}_{0})$ and $\alpha_k$
		\State $\mathbr{\theta}^{j}_{k+1} = \mathbr{\theta}^{j}_{k} + \alpha_k {\mathcal{G}}(\mathbr{\theta}^{j}_{k})^{-1} \nabla_{\mathbr{\theta}_{k}^{j}}{\mathcal{L}}(\mathbr{\theta}^{j}_{k} / \mathbr{\theta}^{j}_{0})$
\EndFor
	\State $\mathbr{\theta}^{j+1}_{0} = \mathbr{\theta}^{j}_{k}$
\EndFor
\end{algorithmic}
  \end{algorithm}
\end{minipage}%
\hfill
\begin{minipage}[t]{0.48\textwidth}
  \vspace{0pt}
  \begin{algorithm}[H]
  \small
  \caption{Parameter-based POIS}
  \label{alg:pbPOIS}
    \begin{algorithmic}
\State Initialize $\mathbr{\rho}^{0}_{0}$ arbitrarily
\For{$j = 0,1,2,..., \text{ until convergence}$}
	\State Sample $N$ policy parameters $\vtheta_i^j$ from $\nu_{\mathbr{\rho}_{0}^j}$
	\State Collect a trajectory with each $\pi_{\vtheta_i^j}$
	\For{$k = 0,1,2,..., \text{ until convergence}$}
		\State Compute $\mathcal{G}(\mathbr{\rho}^{j}_{k})$, $\nabla_{\mathbr{\rho}_{k}^{j}}{\mathcal{L}}(\mathbr{\rho}^{j}_{k} /\mathbr{\rho}^{j}_{0})$ and $\alpha_k$
		\State $\mathbr{\rho}^{j}_{k+1} =\mathbr{\rho}^{j}_{k} + \alpha_k \mathcal{G}(\mathbr{\rho}^{j}_{k})^{-1} \nabla_{\mathbr{\rho}_{k}^{j}}{\mathcal{L}}(\mathbr{\rho}^{j}_{k} / \mathbr{\rho}^{j}_{0})$
\EndFor
	\State $\mathbr{\rho}^{j+1}_{0} = \mathbr{\rho}^{j}_{k}$
\EndFor
\end{algorithmic}
  \end{algorithm}
\end{minipage}
\end{figure}

\section{Policy Optimization via Importance Sampling}
\label{sec:policyOptimization}
In this section, we discuss how to customize the bound provided in Theorem~\ref{thr:bound} for policy optimization, developing a novel model-free actor-only policy search algorithm, named \emph{Policy Optimization via Importance Sampling} (POIS). We propose two versions of POIS: \emph{Action-based POIS} (A-POIS), which is based on a policy gradient approach, and \emph{Parameter-based POIS} (P-POIS), which adopts the PGPE framework. A more detailed description of the implementation aspects is reported in Appendix~\ref{apx:impl}. 

\subsection{Action-based POIS}
In Action-based POIS (A-POIS) we search for a policy that maximizes the performance index $J_D(\mathbr{\theta})$ within a parametric space $\Pi_\Theta = \{ \pi_{\mathbr{\theta}} : \mathbr{\theta} \in \Theta \subseteq \mathbb{R}^p \}$ of stochastic differentiable policies. 
In this context, the behavioral (resp. target) distribution $Q$ (resp. $P$) becomes the distribution over trajectories $p(\cdot|\mathbr{\theta})$ (resp. $p(\cdot|\mathbr{\theta}')$) induced by the behavioral policy $\pi_{\mathbr{\theta}}$ (resp. target policy $\pi_{\mathbr{\theta}'}$) and $f$ is the trajectory return $R(\tau)$ which is uniformly bounded as $|R(\tau)|\le R_{\max}\frac{1-\gamma^{H}}{1-\gamma}$.\footnote{When $\gamma\rightarrow 1$ the bound becomes $HR_{\max}$.} 
The surrogate loss function cannot be directly optimized via gradient ascent since computing $d_\alpha \left( p(\cdot|{\mathbr{\theta}'}) \| p(\cdot|{\mathbr{\theta}}) \right)$ requires the approximation of an integral over the trajectory space and, for stochastic environments, to know the transition model $P$, which is unknown in a model-free setting. 
Simple bounds to this quantity, like $d_\alpha \left( p(\cdot|{\mathbr{\theta}'}) \| p(\cdot|{\mathbr{\theta}}) \right) \le \sup_{s \in \mathcal{S}} d_\alpha \left( \pi_{\mathbr{\theta}'} (\cdot|s) \| \pi_{\mathbr{\theta}} (\cdot|s) \right)^H$, besides being hard to compute due to the presence of the supremum, are extremely conservative since the \Renyi divergence is raised to the horizon $H$. 
We suggest the replacement of the \Renyi divergence with an estimate $\widehat{d}_2\left(p(\cdot|\mathbr{\theta}') \| p(\cdot|\mathbr{\theta}) \right) = \frac{1}{N} \sum_{i=1}^N \prod_{t=0}^{H-1} d_{2} \left( \pi_{\mathbr{\theta}'}(\cdot|s_{\tau_i,t}) \| \pi_{\mathbr{\theta}}(\cdot|s_{\tau_i,t})\right)$ defined only in terms of the policy \Renyi divergence (see Appendix~\ref{apx:estRenyi} for details). Thus, we obtain the following surrogate objective:
\begin{equation}
	\mathcal{L}_{\lambda}^{\mathrm{A-POIS}}(\mathbr{\theta'}/\mathbr{\theta}) = \frac{1}{N} \sum_{i=1}^N w_{\mathbr{\theta}'/\mathbr{\theta}}(\tau_i) R(\tau_i) - \lambda \sqrt{ \frac{\widehat{d}_2 \left( p(\cdot|\mathbr{\theta}') \| p(\cdot|\mathbr{\theta}) \right) }{N}  },
\end{equation}
where $w_{\mathbr{\theta}'/\mathbr{\theta}}(\tau_i) = \frac{p(\tau_i|\mathbr{\theta}')}{p(\tau_i|\mathbr{\theta})} = \prod_{t=0}^{H-1} \frac{\pi_{\mathbr{\theta}'}(a_{\tau_i,t}|s_{\tau_i,t})}{\pi_{\mathbr{\theta}}(a_{\tau_i,t}|s_{\tau_i,t})}$. We consider the case in which $\pi_{\mathbr\theta}(\cdot|s)$ is a Gaussian distribution over actions whose mean depends on the state and whose covariance is state-independent and diagonal: $\mathcal{N}(u_{\mathbr{\mu}}(s), \mathrm{diag}(\mathbr{\sigma}^2))$, where $\mathbr{\theta} = (\mathbr{\mu}, \mathbr{\sigma})$. The learning process mixes online and offline optimization. At each online iteration $j$, a dataset of $N$ trajectories is collected by executing in the environment the current policy $\pi_{\mathbr{\theta}^j_{0}}$. These trajectories are used to optimize the surrogate loss function $\mathcal{L}_{\lambda}^{\mathrm{A-POIS}}$. At each offline iteration $k$, the parameters are updated via gradient ascent: $\mathbr{\theta}_{k+1}^j = \mathbr{\theta}_k^j + \alpha_k {\mathcal{G}}(\mathbr{\theta}^{j}_{k})^{-1} {\nabla}_{\mathbr{\theta}_{k}^{j}}{\mathcal{L}}(\mathbr{\theta}^{j}_{k} / \mathbr{\theta}^{j}_{0})$, where $\alpha_k > 0$ is the step size which is chosen via line search (see Appendix~\ref{apx:lineSearch}) and ${\mathcal{G}}(\mathbr{\theta}^{j}_{k}) $ is a positive semi-definite matrix (\eg ${\mathcal{F}}(\mathbr{\theta}_{k}^j)$, the FIM, for natural gradient)\footnote{The FIM needs to be estimated via importance sampling as well, as shown in Appendix~\ref{apx:implFIM}.}. The pseudo-code of POIS is reported in Algorithm~\ref{alg:cbPOIS}. 

\subsection{Parameter-based POIS}
In the Parameter-based POIS (P-POIS) we again consider a parametrized policy space $\Pi_\Theta = \{ \pi_{\mathbr{\theta}} : \mathbr{\theta} \in \Theta \subseteq \mathbb{R}^p \}$, but $\pi_{\mathbr{\theta}}$ needs not be differentiable. The policy parameters $\vtheta$ are sampled at the beginning of each episode from a parametric hyperpolicy $\nu_{\mathbr{\rho}}$ selected in a parametric space ${\mathcal{N}_\mathcal{P} = \{\nu_{\vrho} : \vrho\in \mathcal{P} \subseteq \mathbb{R}^r\}}$. The goal is to learn the \textit{hyperparameters} $\vrho$ so as to maximize $J_D(\vrho)$. In this setting, the distributions $Q$ and $P$ of Section~\ref{sec:optimization} correspond to the behavioral $\nu_{\mathbr{\rho}}$ and target $\nu_{\mathbr{\rho}'}$ hyperpolicies, while $f$ remains the trajectory return $R(\tau)$.
The importance weights~\cite{zhao2013efficient} must take into account all sources of randomness, derived from sampling a policy parameter $\mathbr{\theta}$ and a trajectory $\tau$: $w_{\vrho'/\vrho}(\vtheta) = \frac{\nu_{\vrho'}(\vtheta)p(\tau\vert\vtheta)}{\nu_{\vrho}(\vtheta)p(\tau\vert\vtheta)}= 
\frac{\nu_{\vrho'}(\vtheta)}{\nu_{\vrho}(\vtheta)}$.
In practice, a Gaussian hyperpolicy $\nu_{\vrho}$ with diagonal covariance matrix is often used, \ie $\mathcal{N}(\mathbr{\mu}, \mathrm{diag}(\mathbr{\sigma}^2))$ with $\mathbr{\rho} = (\mathbr{\mu}, \mathbr{\sigma})$. The policy is assumed to be deterministic: $ \pi_{\vtheta}(a|s) = \delta(a - u_{\vtheta}(s))$, where $u_{\mathbr{\theta}}$ is a deterministic function of the state $s$~\cite[\eg][]{sehnke2010parameter, gruttnermulti}.
A first advantage over the action-based setting is that the distribution of the importance weights is entirely known, as it is the ratio of two Gaussians and the \Renyi divergence $d_{2}(\nu_{\mathbr{\rho}'}\|\nu_{\mathbr{\rho}})$ can be computed exactly~\cite{burbea1984convexity} (see Appendix~\ref{apx:IS}). This leads to the following surrogate objective:
\begin{align}
	\mathcal{L}_{\lambda}^{\mathrm{P-POIS}}(\mathbr{\rho}'/\mathbr{\rho}) = \frac{1}{N} \sum_{i=1}^N w_{\mathbr{\rho}'/\mathbr{\rho}}(\mathbr{\theta}_i) R(\tau_i) - \lambda \sqrt{\frac{d_2 \left( \nu_{\mathbr{\rho}'} \| \nu_{\mathbr{\rho}} \right) }{N}},
\end{align}
where each trajectory $\tau_i$ is obtained by running an episode with action policy $\pi_{\vtheta_i}$, and the corresponding policy parameters $\vtheta_i$ are sampled independently from hyperpolicy $\nu_{\vrho}$, at the beginning of each episode. The hyperpolicy parameters are then updated offline as $\mathbr{\rho}_{k+1}^j = \mathbr{\rho}_k^j + \alpha_k \mathcal{G}(\mathbr{\rho}^{j}_{k})^{-1} {\nabla}_{\mathbr{\rho}_{k}^{j}}{\mathcal{L}}(\mathbr{\rho}^{j}_{k} / \mathbr{\rho}^{j}_{0})$ (see Algorithm \ref{alg:pbPOIS} for the complete pseudo-code).
A further advantage \wrt the action-based case is that the FIM $\mathcal{F(\vrho)}$ can be computed exactly, and it is diagonal in the case of a Gaussian hyperpolicy with diagonal covariance matrix, turning a problematic inversion into a trivial division (the FIM is block-diagonal in the more general case of a Gaussian hyperpolicy, as observed in \cite{miyamae2010natural}). This makes natural gradient much more enticing for P-POIS.

\begin{figure}[t!]
\centering
\footnotesize
\begin{subfigure}[c]{.33\textwidth}
  \raggedleft
  \setlength{\tabcolsep}{3pt}
  \begin{tabular}{lcccc}
    \toprule
    \scriptsize Task    & \scriptsize A-POIS     & \scriptsize  P-POIS  & \scriptsize  TRPO  & \scriptsize  PPO\\
    \midrule
    (a) & 0.4 & 0.4 & 0.1 & 0.01\\ 
    (b) & 0.1 & 0.1 & 0.1 & 1\\
    (c) & 0.7 & 0.2 & 1 & 1\\
    (d) & 0.9 & 1 & 0.01 & 1\\
    (e) & 0.9 & 0.8 & 0.01 & 0.01\\
    \bottomrule
    \vspace{0.17cm}
  \end{tabular}
  \hspace{0.4cm} \includegraphics[scale=1]{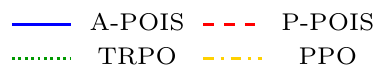}
  \vfill
\end{subfigure}
\begin{subfigure}[l]{.33\textwidth}
  \centering
  \includegraphics[scale=1]{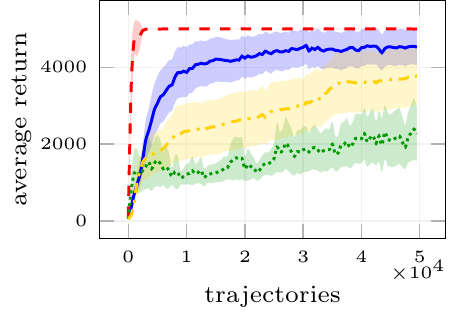}
  \vspace{-0.3cm}
   \caption{Cartpole}
\end{subfigure}%
\hfill
\begin{subfigure}[l]{.33\textwidth}
  \centering
  \includegraphics[scale=1]{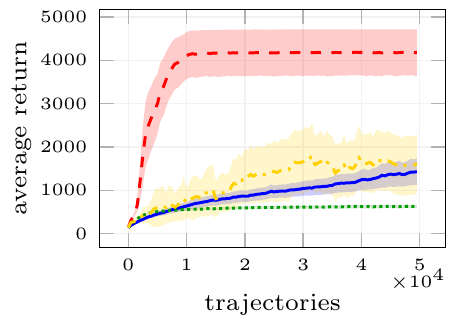}
  \vspace{-0.3cm}
  \caption{Inverted Double Pendulum}
\end{subfigure}%
\hfill
\begin{subfigure}[l]{.33\textwidth}
\vspace{0.4cm}
  \centering
  \includegraphics[scale=1]{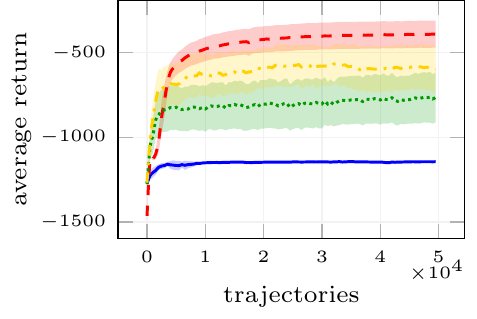}
  \vspace{-0.3cm}
   \caption{Acrobot}
\end{subfigure}
\hfill
\begin{subfigure}[l]{.33\textwidth}
	\vspace{0.4cm}
  \centering
  \includegraphics[scale=1]{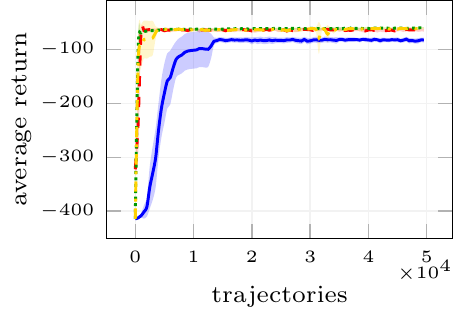}
  \vspace{-0.3cm}
  \caption{Mountain Car}
\end{subfigure}%
\hfill
\begin{subfigure}[l]{.33\textwidth}
\vspace{0.4cm}
  \centering
  \includegraphics[scale=1]{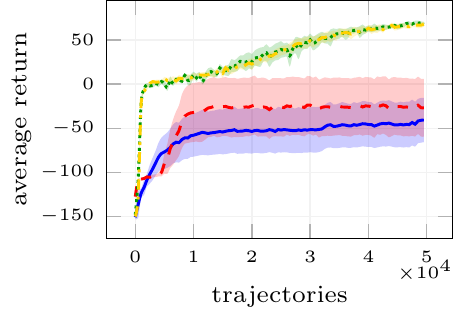}
  \vspace{-0.3cm}
   \caption{Inverted Pendulum}
\end{subfigure}%
\hfill
\caption{Average return as a function of the number of trajectories for A-POIS, P-POIS and TRPO with \emph{linear policy} (20 runs, 95\% c.i.). The table reports the best hyperparameters found ($\delta$ for POIS and the step size for TRPO and PPO).}
\label{fig:plotLin}
\end{figure}

\section{Experimental Evaluation}
\label{sec:experimental}
In this section, we present the experimental evaluation of POIS in its two flavors (action-based and parameter-based). We first provide a
set of empirical comparisons on classical continuous control tasks with linearly parametrized policies; we then show how POIS can be also adopted for learning deep neural policies. In all experiments, for the A-POIS we used the IS estimator, while for P-POIS we employed the SN estimator. All experimental details are provided in Appendix~\ref{apx:exp}.

\subsection{Linear Policies}
Linear parametrized Gaussian policies proved their ability to scale on complex control tasks~\cite{rajeswaran2017towards}. In this section, we compare the learning performance of A-POIS and P-POIS against TRPO~\cite{schulman2015trust} and PPO~\cite{schulman2017proximal} on classical continuous control benchmarks~\cite{duan2016benchmarking}. 
In Figure~\ref{fig:plotLin}, we can see that both versions of POIS are able to significantly outperform both TRPO and PPO in the Cartpole environments, especially the P-POIS. In the Inverted Double Pendulum environment the learning curve of P-POIS is remarkable while A-POIS displays a behavior comparable to PPO. In the Acrobot task, P-POIS displays a
better performance \wrt TRPO and PPO, but A-POIS does not keep up. In Mountain Car, we see yet another behavior: the learning curves of TRPO, PPO and P-POIS are almost one-shot (even if PPO shows a small instability), while A-POIS fails to display such a fast convergence. Finally, in the Inverted Pendulum environment, TRPO and PPO outperform both versions of POIS. This example highlights a limitation of our approach. Since POIS performs an importance sampling procedure at trajectory level, it cannot assign credit to good actions in bad trajectories. On the contrary, weighting each sample, TRPO and PPO are able also to exploit good trajectory segments. In principle, this problem can be mitigated in POIS by resorting to \emph{per-decision importance sampling}~\cite{precup2000eligibility}, in which the weight is assigned to individual rewards instead of trajectory returns.
Overall, POIS displays a performance comparable with TRPO and PPO across the tasks. In particular, P-POIS displays a better performance \wrt A-POIS. However, this ordering is not maintained when moving to more complex policy architectures, as shown in the next section.

In Figure~\ref{fig:plotDelta} we show, for several metrics, the behavior of A-POIS when changing the $\delta$ parameter in the Cartpole environment. We can see that when $\delta$ is small (\eg $0.2$), the Effective Sample Size ($\mathrm{ESS}$) remains large and, consequently, the variance of the importance weights ($\Var[w]$) is small. This means that the penalization term in the objective function discourages the optimization process from selecting policies which are far from the behavioral policy. As a consequence, the displayed behavior is very conservative, preventing the policy from reaching the optimum. On the contrary, when $\delta$ approaches 1, the ESS is smaller and the variance of the weights tends to increase significantly. Again, the performance remains suboptimal as the penalization term in the objective function is too light. The best behavior is obtained with an intermediate value of $\delta$, specifically $0.4$.

\subsection{Deep Neural Policies}\label{sec:dnp}
In this section, we adopt a deep neural network (3 layers: 100, 50, 25 neurons each) to represent the policy. The experiment setup is fully compatible with the classical benchmark~\cite{duan2016benchmarking}. While A-POIS can be directly applied to deep neural networks, P-POIS exhibits some critical issues. A highly dimensional hyperpolicy (like a Gaussian from which the weights of an MLP policy are sampled) can make $d_2(\nu_{\mathbr{\rho}'}\|\nu_{\mathbr{\rho}})$ extremely sensitive to small parameter changes, leading to over-conservative updates.\footnote{This curse of dimensionality, related to $\dim(\vtheta)$, has some similarities with the dependence of the \Renyi divergence on the actual horizon $H$ in the action-based case.} A first practical variant comes from the insight that $d_2(\nu_{\vrho'}\|\nu_{\vrho}) / N$ is the inverse of the effective sample size, as reported in Equation \ref{eq:ess}. We can obtain a less conservative (although approximate) surrogate function by replacing it with $1/\widehat{\ess} (\nu_{\mathbr{\rho}'}\|\nu_{\mathbr{\rho}})$. Another trick is to model the hyperpolicy as a set of independent Gaussians, each defined over a disjoint subspace of $\Theta$ (implementation details are provided in Appendix~\ref{apx:implPGPE}).
In Table~\ref{tab:benchmark}, we augmented the results provided in~\cite{duan2016benchmarking} with
the performance of POIS for the considered tasks. We can see that A-POIS is able to reach an overall behavior comparable with the best of the action-based algorithms, approaching TRPO and beating DDPG. Similarly, P-POIS exhibits a performance similar to CEM~\cite{szita2006learning}, the best performing among the parameter-based methods. The complete results are reported in Appendix~\ref{apx:exp}.

\begin{figure}[t!]
\centering
\footnotesize
\begin{subfigure}[l]{.33\textwidth}
\vspace{0.4cm}
  \centering
  \includegraphics[scale=1]{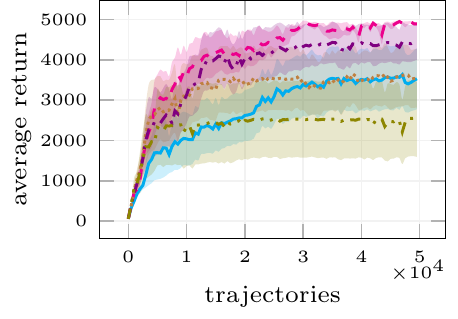}
\end{subfigure}
\hfill
\begin{subfigure}[l]{.33\textwidth}
	\vspace{0.4cm}
  \centering
  \includegraphics[scale=1]{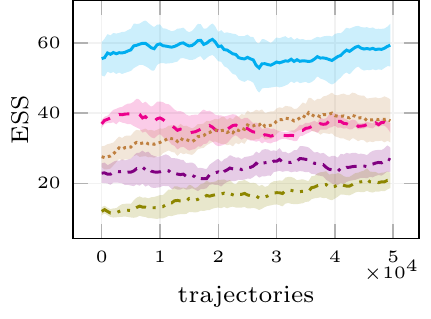}
\end{subfigure}%
\hfill
\begin{subfigure}[l]{.33\textwidth}
\vspace{0.4cm}
  \centering
  \includegraphics[scale=1]{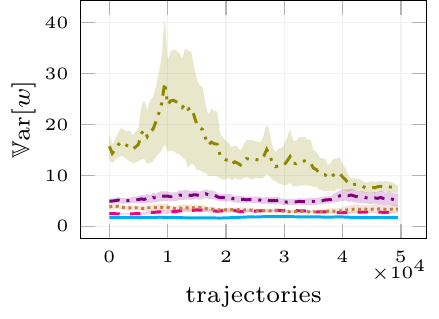}
\end{subfigure}%
\hfill
\centering
 \includegraphics[scale=1]{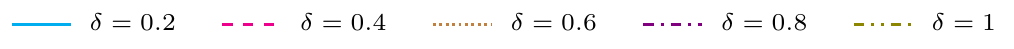}
\caption{Average return, Effective Sample Size (ESS) and variance of the importance weights ($\Var[w]$) as a function of the number of trajectories for A-POIS for different values of the parameter $\delta$ in the Cartpole environment (20 runs, 95\% c.i.).}
\label{fig:plotDelta}
\end{figure}

\setlength{\tabcolsep}{7pt}
\begin{table}[t]
  \caption{Performance of POIS compated with~\cite{duan2016benchmarking} on \emph{deep neural policies} (5 runs, 95\% c.i.). In \textbf{bold}, the performances that are not statistically significantly different from the best algorithm in each task.}
  \label{tab:benchmark}
  \centering
  \footnotesize
  \begin{tabular}{lcccc}
    \toprule
        & Cart-Pole     &   & Double Inverted  &   \\
    	Algorithm		& Balancing & Mountain Car & Pendulum & Swimmer \\
    \midrule
    REINFORCE & $ 4693.7 \pm 14.0 $  & $-67.1 \pm 1.0 $ & $4116.5 \pm 65.2 $ & $92.3 \pm 0.1$   \\
    TRPO & $ \mathbf{4869.8 \pm 37.6} $  & $\mathbf{ -61.7 \pm 0.9 }$ & $\mathbf{4412.4 \pm 50.4}$ & $\mathbf{96.0 \pm 0.2}$      \\
    DDPG & $ 4634.4 \pm 87.6 $  & $ -288.4 \pm 170.3$ & $2863.4 \pm 154.0$ & $85.8 \pm 1.8$     \\
    \rowcolor{blue!20}
    A-POIS & $ \mathbf{4842.8 \pm 13.0} $  & $ -63.7 \pm 0.5$ & $\mathbf{4232.1 \pm 189.5 }$ & $88.7 \pm 0.55$    \\
     \hdashline[3pt/2pt]
    CEM & $ 4815.4 \pm 4.8 $  & $ -66.0 \pm 2.4 $ & $2566.2 \pm 178.9$ & $68.8 \pm 2.4$     \\
    \rowcolor{red!20}
    P-POIS & $ 4428.1 \pm 138.6 $  & $ -78.9 \pm 2.5 $ & $3161.4 \pm 959.2$  & $76.8 \pm 1.6$    \\
    \bottomrule
  \end{tabular}
\end{table}

\section{Discussion and Conclusions}
\label{sec:discussion}
In this paper, we presented a new actor-only policy optimization algorithm, POIS, which
alternates online and offline optimization in order to efficiently exploit the collected trajectories, and can be used in combination with
action-based and parameter-based exploration. In contrast to the state-of-the-art algorithms,
POIS has a strong theoretical grounding, since its surrogate objective function
derives from a statistical bound on the estimated performance, that is able to capture
the uncertainty induced by importance sampling. The experimental evaluation
showed that POIS, in both its versions (action-based and parameter-based), is 
able to achieve a performance comparable with TRPO, PPO and other classical algorithms on continuous control tasks.
Natural extensions of POIS could focus on employing per-decision importance sampling, adaptive batch size, and trajectory reuse. Future work also includes scaling POIS to high-dimensional tasks and highly-stochastic environments.
We believe that this work represents a valuable starting point for a deeper
understanding of modern policy optimization and for the development of effective and scalable
policy search methods.


\newpage
\subsubsection*{Acknowledgments}
The study was partially funded by Lombardy Region (Announcement PORFESR 2014-2020).\\
F.F. was partially funded through ERC Advanced Grant (no: 742870).

\bibliographystyle{plain}
{\small \bibliography{biblio}}

\clearpage
\appendix

\section*{Index of the Appendix}
In the following, we briefly recap the contents of the Appendix.
\begin{itemize}[leftmargin=*, label={--}]
	\item Appendix~\ref{apx:relatedWorks} shows a more detailed comparison of POIS with the policy-search algorithms. Table~\ref{tab:comp} summarizes some features of the considered methods.
	\item Appendix~\ref{apx:proofs} reports all proofs and derivations.
	\item Appendix~\ref{apx:IS} provides an analysis of the distribution of the importance weights in the case of univariate Gaussian behavioral and target distributions.
	\item Appendix~\ref{apx:SN} shows some bounds on bias and variance for the self-normalized importance sampling estimator and provides a high confidence bound.
	\item Appendix~\ref{apx:impl} illustrates some implementation details of POIS, in particular line search algorithms, estimation of the \Renyi divergence, computation of the FIM and practical versions of P-POIS.
	\item Appendix~\ref{apx:exp} provides the hyperparameters used in the experiments and further results.
\end{itemize}

\section{Related Works}
\label{apx:relatedWorks}
Policy optimization algorithms can be classified according to different dimensions (Table~\ref{tab:comp}). It is by now established, in the policy-based RL community, that effective algorithms, either on-policy or off-policy, should account for the variance of the gradient estimate. Early attempts, in the
class of action-based algorithms, are the usage of a baseline to reduce the estimated gradient variance without introducing bias~\cite{baxter2001infinite, peters2008reinforcement}. A similar rationale is at the basis of actor-critic architectures~\cite{konda2000actor, sutton2000policy, peters2008natural}, in which
an estimate of the value function is used to reduce uncertainty. Baselines are typically constant (REINFORCE), time-dependent (G(PO)MDP) or state-dependent (actor-critic), but these approaches have been recently extended to account for action-dependent baselines~\cite{tucker2018mirage, wu2018variance}. Even though parameter-based algorithms are, by nature, affected by smaller variance \wrt action-based ones, similar baselines can be derived~\cite{zhao2011analysis}. A first dichotomy in the class of policy-based algorithms comes when considering
the minimal unit used to compute the gradient. \emph{Episode-based} (or episodic) approaches~\cite[\eg][]{williams1992simple, baxter2001infinite} perform the gradient estimation by averaging the gradients of each episode which need to have a finite horizon. On the contrary, \emph{step-based} approaches~\cite[\eg][]{schulman2015trust, schulman2017proximal, lillicrap2015continuous}, derived from the Policy Gradient Theorem~\cite{sutton2000policy}, can estimate the gradient by averaging over timesteps. The latter requires a function approximator
(a critic) to estimate the Q-function, or directly the advantage function~\cite{schulman2015high}. When coming to the on/off-policy dichotomy, the previous distinction has a relevant impact. Indeed, episode-based approaches need to perform importance sampling on trajectories, thus the importance weights are the products of policy ratios for all executed actions within a trajectory, whereas step-based algorithms need just to weight each sample with the corresponding policy ratio. The latter case helps to keep the value of the importance weights close to one, but the need to have a critic prevents from a complete analysis of the uncertainty since the bias/variance injected by the critic is hard to compute~\cite{konda2000actor}. Moreover, in the off-policy scenario, it is necessary to control some notion of dissimilarity between the behavioral and target policy, as the variance increases when moving too far. This is the case of TRPO~\cite{schulman2015trust}, where the regularization constraint based on the Kullback-Leibler divergence helps controlling the importance weights but originates from an exact bound on the performance improvement. Intuitively, the same rationale applies to the truncation of the importance weights, employed by PPO, that avoids performing too large steps in the policy space. Nevertheless, the step size in TRPO and the truncation range $\epsilon$ in PPO are just hyperparameters and have a limited statistical meaning. On the contrary, other actor-critic architectures have been proposed including also experience replay methods, like~\cite{wang2016sample} in which the importance weights are truncated, but the method is able to account for the injected bias. The authors propose to keep a running mean of the best
policies seen so far to avoid a hard constraint on the policy dissimilarity.
Differently from these methods, POIS directly models the uncertainty due to the importance sampling procedure. The bound in Theorem~\ref{thr:bound} introduces the unique hyperparameter $\delta$ which has a precise statistical meaning as confidence level. The optimal value of $\delta$ (like the step size in TRPO and $\epsilon$ in PPO) is task-dependent and might vary during the learning procedure. Furthermore, POIS is an episode-based approach in which the importance weights account for the whole trajectory at once; this might prevent from assigning credit to valuable subtrajectories (like in the case of Inverted Pendulum, see Figure~\ref{fig:plotLin}). A possible solution is to resort to per-decision importance sampling~\cite{precup2000eligibility}.

\begin{landscape}
\renewcommand{\arraystretch}{1.5}
\begin{table}[t]
  \caption{Comparison of some policy optimization algorithms according to different dimensions. For brevity, we will indicate with $w_{\mathbr{\theta}'/\mathbr{\theta}}(a_t|s_t) = \frac{\pi_{\mathbr{\theta}'}(a_t|s_t)}{\pi_{\mathbr{\theta}}(a_t|s_t)}$. For episode-based algorithms we will indicate with $\widehat{\ev}_{\tau \sim \mathbr{\theta}}$ the empirical average over trajectories collected with $\pi_{\mathbr{\theta}}$. For step-based algorithms $\widehat{\ev}_{t \sim \mathbr{\theta}}$ is the empirical average collecting samples with $\pi_{\mathbr{\theta}}$. For parameter-based algorithms we indicate with $\widehat{\ev}_{\mathbr{\theta} \sim \mathbr{\rho}, \tau \sim \mathbr{\theta}}$ the empirical expectation taken \wrt policy parameter $\mathbr{\theta}$ sampled from the hyperpolicy $\nu_{\mathbr{\rho}}$ and trajectory $\tau$ collected with $\pi_{\mathbr{\theta}}$. For the actor-critic architectures, $\widehat{Q}$ and $\widehat{A}$ are the estimated Q-function and advantage function.}
  \label{tab:comp}
  \centering
  \small
  \setlength{\extrarowheight}{0.04cm}
  \begin{tabularx}{\linewidth}{m{2.8cm}>{\centering}m{2.8cm}>{\centering}m{2.5cm}m{7.3cm}>{\centering}m{2cm}m{2.5cm}} 
    \toprule
    Algorithm   & Action/Parameter based & On/Off policy & Optimization problem & Critic & Timestep/Trajectory based\\
    \midrule
    REINFORCE/ G(PO)MDP~\cite{williams1992simple, baxter2001infinite} & action-based & on-policy & $\max \widehat{\ev}_{\tau \sim \mathbr{\theta}} \left[ R(\tau) \right]$  & No & episode-based\\
    TRPO~\cite{schulman2015trust} & action-based & on-policy & $\max \widehat{\ev}_{t \sim \mathbr{\theta}} \left[w_{\mathbr{\theta}'/\mathbr{\theta}}(a_t|s_t) \widehat{A}(s_t,a_t) \right]$ \newline \quad s.t. $\widehat{\ev}_{t \sim \mathbr{\theta}} \left[ D_{\mathrm{KL}}(\pi_{\mathbr{\theta}'}(\cdot|s_t) \| \pi_{\mathbr{\theta}}(\cdot|s_t))\right] \le \delta $ & Yes & step-based\\
    PPO~\cite{schulman2017proximal} & action-based & on/off-policy & $\begin{aligned} &\max \widehat{\ev}_{t \sim \mathbr{\theta}} \Big[ \min \Big\{ w_{\mathbr{\theta}'/\mathbr{\theta}}(a_t|s_t)\widehat{A}(s_t,a_t), \\ & \quad \mathrm{clip} \left( w_{\mathbr{\theta}'/\mathbr{\theta}}(a_t|s_t), 1-\epsilon, 1+\epsilon  \right) \widehat{A}(s_t,a_t) \Big\} \Big]\end{aligned}$  & Yes & step-based\\
    DDPG~\cite{lillicrap2015continuous} & action-based & off-policy & $\max  \widehat{\ev}_{t \sim \mathbr{\theta}} \left[ \pi_{\mathbr{\theta}'}(a_t|s_t) \widehat{Q}(s_t,a_t) \right]$  & Yes & step-based\\
     REPS~\cite{peters2010relative}\footnotemark & action-based & on-policy & $\max \widehat{\ev}_{t \sim \mathbr{\theta}} \left[ R(s_t, a_t) \right]$  \newline \quad s.t. $\widehat{\ev}_{t \sim \mathbr{\theta}} \left[D_{\mathrm{KL}} (d_{\mu}^{\pi_{\mathbr{\theta}'}}(s_t,a_t) \| d_{\mu}^{\pi_{\mathbr{\theta}}}(s_t,a_t))\right] \le \delta $  & Yes & step-based\\
     RWR~\cite{peters2007reinforcement} & action-based & on-policy & $\max \widehat{\ev}_{t \sim \mathbr{\theta}} \left[ \beta \exp \left( -\beta R(s_t, a_t) \right) \right]$  & No & step-based\\
    \rowcolor{blue!20}
    A-POIS & action-based & on/off-policy & $\max \widehat{\ev}_{\tau \sim \mathbr{\theta}} \left[ w_{\mathbr{\theta}'/\mathbr{\theta}}(\tau) R(\tau) \right] - \lambda \sqrt{\widehat{d_2}(p(\cdot|{\mathbr{\theta}'}) \|{p(\cdot|{\mathbr{\theta}}}))/N}$  & No & episode-based\\
     \hdashline[3pt/2pt]
    PGPE~\cite{sehnke2008policy} & parameter-based & on-policy & $\max \widehat{\ev}_{\mathbr{\theta} \sim {\mathbr{\rho}}, \tau \sim \mathbr{\theta}} \left[ R(\tau) \right]$  & No & episode-based\\
    IW-PGPE~\cite{zhao2013efficient} & parameter-based & on/off-policy & $\max  \widehat{\ev}_{\mathbr{\theta} \sim \mathbr{\rho}, \tau \sim \mathbr{\theta}} \left[ w_{\mathbr{\rho}'/\mathbr{\rho}}(\mathbr{\theta}) R(\tau) \right]$  & No & episode-based\\
    \rowcolor{red!20}
    P-POIS & parameter-based & on/off-policy & $\max \widehat{\ev}_{\mathbr{\theta} \sim \mathbr{\rho}, \tau \sim \mathbr{\theta}} \left[w_{\mathbr{\rho}'/\mathbr{\rho}}(\mathbr{\theta}) R(\tau) \right] - \lambda\sqrt{{d_2}(\nu_{\mathbr{\rho}'} \|{\nu_{\mathbr{\rho}}})/N}$  & No & episode-based\\
    \bottomrule
  \end{tabularx}
\end{table}
\footnotetext{We indicate with $d_{\mu}^{\pi_{\mathbr{\theta}}}(s_t,a_t)$ the state-action occupancy~\cite{sutton2000policy}.}

\end{landscape}

\section{Proofs and Derivations}
\label{apx:proofs}
\boundVariance*
\begin{proof}
	From the fact that $x_i$ are i.i.d. we can write:
	\begin{align*}
		\Var_{\mathbr{x} \sim Q}  \left[ \widehat{\mu}_{P/Q} \right] & \le \frac{1}{N} \Var_{x_1 \sim Q} \left[  \frac{p(x_1)}{q(x_1)} f(x_1)  \right] \le \frac{1}{N} \ev_{x_1 \sim Q} \left[ \left( \frac{p(x_1)}{q(x_1)} f(x_1) \right)^2 \right] \\
		& \le \frac{1}{N} \|f\|_{\infty}^2 \ev_{x_1 \sim Q} \left[ \left( \frac{p(x_1)}{q(x_1)}  \right)^2 \right] =  \frac{1}{N} \|f\|_{\infty}^2 d_2 \left( P \| Q \right).
	\end{align*}
\end{proof}

\bound*
\begin{proof}
	We start from Cantelli's inequality applied on the random variable $\widehat{\mu}_{P/Q} =  \frac{1}{N} \sum_{i=1}^N w_{P/Q}(x_i) f(x_i) $:
	\begin{equation}
		\Pr \left( \widehat{\mu}_{P/Q} - \ev_{x \sim P} \left[ f(x) \right] \ge \lambda \right) \le \frac{1}{1+ \frac{\lambda^2}{\Var_{\mathbr{x} \sim Q}  \left[ \widehat{\mu}_{P/Q} \right] }}.
	\end{equation}
	By calling $\delta = \frac{1}{1+ \frac{\lambda^2}{\Var_{\mathbr{x} \sim Q}  \left[ \widehat{\mu}_{P/Q} \right] }}$ and considering the complementary event, we get that with probability at least $1-\delta$ we have:
	\begin{equation}
		\ev_{x \sim P} \left[ f(x) \right] \ge \widehat{\mu}_{P/Q} - \sqrt{\frac{1-\delta}{\delta} \Var_{\mathbr{x} \sim Q}  \left[ \widehat{\mu}_{P/Q} \right] }.
	\end{equation}
	By replacing the variance with the bound in Theorem~\ref{thr:bound} we get the result.
\end{proof}

\riemannMetric*
\begin{proof}
	We need to compute the second-order Taylor expansion of the $\alpha$-\Renyi divergence. We start considering the term:
	\begin{equation}
		I(\mathbr{\omega}') = \int_\mathcal{X} \left( \frac{ p_{\mathbr{\omega}'}(x)}{ p_{\mathbr{\omega}}(x) } \right)^{\alpha} p_{\mathbr{\omega}}(x) \de x = \int_\mathcal{X} { p_{\mathbr{\omega}'}(x)}^{\alpha} { p_{\mathbr{\omega}}(x) }^{1-\alpha}  \de x.
	\end{equation}
	The gradient is given by:
	\begin{align*}
		\nabla_{\mathbr{\omega}'} I(\mathbr{\omega}') & = \int_\mathcal{X} { \nabla_{\mathbr{\omega}'} p_{\mathbr{\omega}'}(x)}^{\alpha} { p_{\mathbr{\omega}}(x) }^{1-\alpha}  \de x =  \alpha \int_\mathcal{X} {  p_{\mathbr{\omega}'}(x)}^{\alpha-1} { p_{\mathbr{\omega}}(x) }^{1-\alpha}  \nabla_{\mathbr{\omega}'} p_{\mathbr{\omega}'}(x)  \de x.
	\end{align*}
	Thus, $\nabla_{\mathbr{\omega}'} I(\mathbr{\omega}') \rvert_{\mathbr{\omega}' = \mathbr{\omega}} = \mathbr{0}$. We now compute the Hessian:
	
\resizebox{.99\linewidth}{!}{
\begin{minipage}{\linewidth}
	\begin{align*}
	\displaystyle 
		\mathcal{H}_{\mathbr{\omega}'} I(\mathbr{\omega}') & = \nabla_{\mathbr{\omega}'} \nabla_{\mathbr{\omega}'}^T I(\mathbr{\omega}') = \alpha  \nabla_{\mathbr{\omega}'} \int_\mathcal{X} {  p_{\mathbr{\omega}'}(x)}^{\alpha-1} { p_{\mathbr{\omega}}(x) }^{1-\alpha}  \nabla_{\mathbr{\omega}'}^T p_{\mathbr{\omega}'}(x)  \de x \\
		& = \alpha  \int_\mathcal{X}   \left( (\alpha - 1) p_{\mathbr{\omega}'}(x)^{\alpha-2}  p_{\mathbr{\omega}}(x) ^{1-\alpha} \nabla_{\mathbr{\omega}'} p_{\mathbr{\omega}'}(x) \nabla_{\mathbr{\omega}'}^T p_{\mathbr{\omega}'}(x) +  p_{\mathbr{\omega}'}(x)^{\alpha-1}  p_{\mathbr{\omega}}(x) ^{1-\alpha}  \mathcal{H}_{\mathbr{\omega}'} p_{\mathbr{\omega}'}(x) \right)\de x. 
	\end{align*}
	\end{minipage}}
	
	Evaluating the Hessian in $\mathbr{\omega}$ we have:
	\begin{align*}
		\mathcal{H}_{\mathbr{\omega}'} I(\mathbr{\omega}')\rvert_{\mathbr{\omega}' = \mathbr{\omega}} & = \alpha(\alpha-1) \int_\mathcal{X} p_{\mathbr{\omega}}(x) ^{-1} \nabla_{\mathbr{\omega}} p_{\mathbr{\omega}}(x) \nabla_{\mathbr{\omega}}^T p_{\mathbr{\omega}}(x) \de x \\
		& = \alpha(\alpha-1) \int_\mathcal{X} p_{\mathbr{\omega}}(x) \nabla_{\mathbr{\omega}} \log p_{\mathbr{\omega}}(x) \nabla_{\mathbr{\omega}}^T \log p_{\mathbr{\omega}}(x) \de x = \alpha (\alpha-1) \mathcal{F}(\mathbr{\omega}).
	\end{align*}
	Now, $D_{\alpha} (p_{\mathbr{\omega}'} \| p_{\mathbr{\omega}}) = \frac{1}{\alpha-1} \log I(\mathbr{\omega}')$. Thus:
	\begin{equation*}
		\nabla_{\mathbr{\omega}'} D_{\alpha} (p_{\mathbr{\omega}'} \| p_{\mathbr{\omega}}) \rvert_{\mathbr{\omega}' = \mathbr{\omega}} = \frac{1}{\alpha-1} \frac{\nabla_{\mathbr{\omega}'} I(\mathbr{\omega}') }{I(\mathbr{\omega}')} \bigg\rvert_{\mathbr{\omega}' = \mathbr{\omega}} = \mathbr{0},
	\end{equation*}
	\begin{align*}
		\mathcal{H}_{\mathbr{\omega}'} D_{\alpha}(p_{\mathbr{\omega}'} \| p_{\mathbr{\omega}}) \rvert_{\mathbr{\omega}'  = \mathbr{\omega}} & = \frac{1}{\alpha-1} \frac{I(\mathbr{\omega}')\mathcal{H}_{\mathbr{\omega}'}I(\mathbr{\omega}') + \nabla_{\mathbr{\omega}'} I(\mathbr{\omega}')\nabla_{\mathbr{\omega}'}^T I(\mathbr{\omega}')} {\left( I(\mathbr{\omega}') \right) ^2} \bigg\rvert_{\mathbr{\omega}' = \mathbr{\omega}} \\
		& = \frac{1}{\alpha-1}  \mathcal{H}_{\mathbr{\omega}'}I(\mathbr{\omega}') \rvert_{\mathbr{\omega}' = \mathbr{\omega}}  = \alpha  \mathcal{F}(\mathbr{\omega}),
	\end{align*}
	having observed that $I(\mathbr{\omega}) = 1$.
	For what concerns the $d_\alpha (p_{\mathbr{\omega}'} \| p_{\mathbr{\omega}})$, we have:
	\begin{align*}
		\nabla_{\mathbr{\omega}'} d_{\alpha} (p_{\mathbr{\omega}'} \| p_{\mathbr{\omega}})\rvert_{\mathbr{\omega}' = \mathbr{\omega}} & = \nabla_{\mathbr{\omega}'} \exp \left(D_{\alpha} (p_{\mathbr{\omega}'} \| p_{\mathbr{\omega}}) \right)\rvert_{\mathbr{\omega}' = \mathbr{\omega}} \\
		& = \exp \left(D_{\alpha} (p_{\mathbr{\omega}'} \| p_{\mathbr{\omega}}) \right) \nabla_{\mathbr{\omega}'} D_{\alpha} (p_{\mathbr{\omega}'} \| p_{\mathbr{\omega}})\rvert_{\mathbr{\omega}' = \mathbr{\omega}} = \mathbr{0},
	\end{align*}
	\begin{align*}
		\mathcal{H}_{\mathbr{\omega}'} & d_{\alpha}(p_{\mathbr{\omega}'} \| p_{\mathbr{\omega}}) \rvert_{\mathbr{\omega}' = \mathbr{\omega}}  = \mathcal{H}_{\mathbr{\omega}'} \exp \left( D_{\alpha}(p_{\mathbr{\omega}'} \| p_{\mathbr{\omega}})  \right) \rvert_{\mathbr{\omega}' = \mathbr{\omega}} \\
		& = \exp \left( D_{\alpha}(p_{\mathbr{\omega}'} \| p_{\mathbr{\omega}}) \right)  \left(\mathcal{H}_{\mathbr{\omega}'} D_{\alpha}(p_{\mathbr{\omega}'} \| p_{\mathbr{\omega}}) + \nabla_{\mathbr{\omega}'} D_{\alpha} (p_{\mathbr{\omega}'} \| p_{\mathbr{\omega}}) \nabla_{\mathbr{\omega}'}^T D_{\alpha} (p_{\mathbr{\omega}'} \| p_{\mathbr{\omega}}) \right) \rvert_{\mathbr{\omega}' = \mathbr{\omega}}\\
		& = \alpha  \mathcal{F}(\mathbr{\omega}).
	\end{align*}
	
\end{proof}

\section{Analysis of the IS estimator}
\label{apx:IS}
In this Appendix, we analyze the behavior of the importance weights when the behavioral and target distributions are Gaussians. We start providing a closed-form expression for the \Renyi divergence between multivariate Gaussian distributions~\cite{burbea1984convexity}. Let $P\sim \mathcal{N}(\mathbr{\mu}_P, \mathbr{\Sigma}_P)$ and $Q \sim  \mathcal{N}(\mathbr{\mu}_Q, \mathbr{\Sigma}_Q)$ and $\alpha \in [0, \infty]$:
\begin{equation}
D_{\alpha}(P \| Q) = \frac{\alpha}{2} (\mathbr{\mu}_P - \mathbr{\mu}_Q)^T \mathbr{\Sigma}_{\alpha}^{-1} (\mathbr{\mu}_P - \mathbr{\mu}_Q) - \frac{1}{2(\alpha-1)} \log \frac{\det(\mathbr{\Sigma}_{\alpha})}{\det(\mathbr{\Sigma}_P)^{1-\alpha} \det(\mathbr{\Sigma}_Q)^{\alpha}},
\end{equation}
where $\mathbr{\Sigma}_{\alpha} = \alpha \mathbr{\Sigma}_Q + (1-\alpha) \mathbr{\Sigma}_P$ under the assumption that $\mathbr{\Sigma}_{\alpha}$ is positive-definite.

From now on, we will focus on univariate Gaussian distributions and we provide a closed-form expression for the importance weights and their probability density function $f_w$. We consider $Q \sim \mathcal{N}(\mu_Q, \sigma^2_Q)$ as behavioral distribution and $P \sim \mathcal{N}(\mu_P, \sigma^2_P)$ as target distribution. We assume that $\sigma_Q^2, \sigma_P^2>0$ and we consider the two cases: unequal variances and equal variances. For brevity, we will indicate
with $w(x)$ the weight $w_{P/Q}(x)$.
\subsection{Unequal variances} When $\sigma_Q^2 \neq \sigma_P^2$, the expression of the importance weights is given by:
\begin{equation}
	w(x) = \frac{\sigma_Q}{\sigma_P} \exp{\left( \frac{1}{2}\frac{(\mu_P - \mu_Q)^2}{\sigma_Q^2 - \sigma_P^2} \right)} \exp{\left( - \frac{1}{2}\frac{\sigma_Q^2 - \sigma_P^2}{\sigma_Q^2\sigma_P^2} \Bigg (x - \frac{\sigma_Q^2\mu_P - \sigma_P^2\mu_Q}{\sigma_Q^2 - \sigma_P^2}\Bigg)^2 \right)},
\end{equation}
for $x \sim Q$. Let us first notice two distinct situations: if $\sigma_Q^2 - \sigma_P^2 > 0$ the weight $w(x)$ is upper bounded by $A = \frac{\sigma_Q}{\sigma_P} \exp{\left( \frac{1}{2}\frac{(\mu_P - \mu_Q)^2}{\sigma_Q^2 - \sigma_P^2} \right)}$, whereas if $\sigma_Q^2 - \sigma_P^2 < 0$, $w(x)$ is unbounded but it admits a minimum of value $A$. Let us investigate the probability density function.
\begin{restatable}[]{prop}{}
 Let $Q \sim \mathcal{N}(\mu_Q, \sigma^2_Q)$ be the behavioral distribution and $P \sim \mathcal{N}(\mu_P, \sigma^2_P)$ be the target distribution, with $\sigma_Q^2\neq\sigma_P^2$. The probability density function of $w(x) = p(x)/q(x)$ is given by:
\begin{equation*}
	f_{w} (y) = \begin{cases}
					\frac{\overline{\sigma}}{ y \sqrt{\pi \log \frac{A}{y}}} \exp \left(-\frac{1}{2} \overline{\mu}^2 \right) \left( \frac{y}{A} \right) ^{ \overline{\sigma}^2} \cosh \left( \overline{\mu}\overline{\sigma} \sqrt{2 \log \frac{A}{y}} \right), & \text{ if } \sigma_Q^2 > \sigma_P^2,\, y \in [0,A],\\
					\frac{\overline{\sigma}}{ y \sqrt{\pi \log \frac{y}{A}}} \exp \left(-\frac{1}{2} \overline{\mu}^2 \right) \left( \frac{A}{y} \right) ^{ \overline{\sigma}^2} \cosh \left( \overline{\mu}\overline{\sigma} \sqrt{2 \log \frac{y}{A}} \right), & \text{ if } \sigma_Q^2 < \sigma_P^2,\,y \in [A,\infty), 
					\end{cases}
\end{equation*}
where $\overline{\mu} = \frac{\sigma_Q}{\sigma_Q^2 - \sigma_P^2} (\mu_P - \mu_Q)$ and $\overline{\sigma}^2 = \frac{\sigma_P^2}{\left| \sigma_Q^2 - \sigma_P^2 \right|}$.
\end{restatable}

\begin{proof}
	We look at $w(x)$ as a function of random variable $x \sim Q$. We introduce the following symbols:
	\begin{equation*}
		m = \frac{\sigma_Q^2\mu_P - \sigma_P^2\mu_Q}{\sigma_Q^2 - \sigma_P^2},\quad
		\tau = \frac{\sigma_Q^2 - \sigma_P^2}{\sigma_Q^2\sigma_P^2}.
	\end{equation*}
	Let us start computing the c.d.f.:
	\begin{align*}
		F_{w}(y) = \Pr \left( w(x) \le y \right) = \Pr \left( A \exp\left( -\frac{1}{2} \tau (x-m)^2 \right) \le y \right) =  \Pr \left( \tau (x-m)^2 \ge -2 \log \frac{y}{A} \right).
	\end{align*}
	We distinguish the two cases according to the sign of $\tau$ and we observe that $x = \mu_Q + \sigma_Q z$ where $z\sim \mathcal{N}(0,1)$:
	\paragraph{$\mathbr{\tau > 0}$:}
	\begin{align*}
		F_{w}(y) &=  \Pr \left(  (x-m)^2 \ge  \frac{2}{\tau} \log \frac{A}{y} \right) \\
		& = \Pr \left ( x \le m - \sqrt{\frac{2}{\tau} \log \frac{A}{y}} \right) + \Pr \left ( x \ge m + \sqrt{\frac{2}{\tau} \log \frac{A}{y}} \right)  \\
		& = \Pr \left ( z \le \frac{m - \mu_Q}{\sigma_Q} - \sqrt{\frac{2}{\tau \sigma_Q^2} \log \frac{A}{y}} \right) + \Pr \left ( z \ge \frac{m - \mu_Q}{\sigma_Q} + \sqrt{\frac{2}{\tau\sigma_Q^2} \log \frac{A}{y}} \right).
	\end{align*}
	We call $\overline{\mu} = \frac{m - \mu_Q}{\sigma_Q} = \frac{\sigma_Q}{\sigma_Q^2 - \sigma_P^2} (\mu_P - \mu_Q)$ and $\overline{\sigma}^2 = \frac{1}{\tau \sigma_Q^2} = \frac{\sigma_P^2}{\sigma_Q^2 - \sigma_P^2}$, thus we have:
	\begin{align*}
		F_{w}(y) &= \Pr \left ( z \le \overline{\mu} - \sqrt{2\overline{\sigma}^2 \log \frac{A}{y}} \right) + \Pr \left ( z \ge \overline{\mu} + \sqrt{2\overline{\sigma}^2 \log \frac{A}{y}} \right) \\
		& = \Phi\left(  \overline{\mu} - \sqrt{2\overline{\sigma}^2 \log \frac{A}{y}} \right) + 1 -  \Phi\left(  \overline{\mu} + \sqrt{2\overline{\sigma}^2 \log \frac{A}{y}} \right),
	\end{align*}
	where $\Phi$ is the c.d.f. of a normal standard distribution. By taking the derivative \wrt $y$ we get the p.d.f.:
	\begin{align*}
		f_{w}(y) &= \frac{\partial F_{w}(y)}{\partial y} = - \sqrt{2\overline{\sigma}^2} \frac{1}{2\sqrt{\log \frac{A}{y}}} \frac{y}{A} \frac{-A}{y^2} \left( \phi\left(  \overline{\mu} - \sqrt{2\overline{\sigma}^2 \log \frac{A}{y}} \right) +  \phi\left(  \overline{\mu} + \sqrt{2\overline{\sigma}^2 \log \frac{A}{y}} \right) \right)  \\
		& = \frac{\sqrt{2} \overline{\sigma}} {2y\sqrt{\log \frac{A}{y}}} \left( \phi\left(  \overline{\mu} - \sqrt{2\overline{\sigma}^2 \log \frac{A}{y}} \right) +  \phi\left(  \overline{\mu} + \sqrt{2\overline{\sigma}^2 \log \frac{A}{y}} \right) \right) \\
		& = \frac{\sqrt{2} \overline{\sigma}}{2y\sqrt{\log \frac{A}{y}} } \frac{1}{\sqrt{2\pi}} \left( \exp\left( -\frac{1}{2} \left( \overline{\mu} - \sqrt{2\overline{\sigma}^2 \log \frac{A}{y}} \right)^2 \right)  + \exp\left( -\frac{1}{2} \left( \overline{\mu} + \sqrt{2\overline{\sigma}^2 \log \frac{A}{y}} \right)^2 \right) \right) \\
		& = \frac{\overline{\sigma}}{y\sqrt{\pi \log \frac{A}{y}} } \exp \left(  -\frac{1}{2} \overline{\mu}^2 \right) \exp \left(  - \overline{\sigma}^2 \log \frac{A}{y} \right) \frac{ \exp\left ( \overline{\mu} \overline{\sigma} \sqrt{2 \log \frac{A}{y}} \right) + \exp\left (- \overline{\mu} \overline{\sigma} \sqrt{2 \log \frac{A}{y}} \right)  }{2} \\
		& = \frac{\overline{\sigma}}{y\sqrt{\pi \log \frac{A}{y}} } \exp \left(  -\frac{1}{2} \overline{\mu}^2 \right) \left( \frac{y}{A} \right)^{\overline{\sigma}^2} \cosh \left( \overline{\mu} \overline{\sigma} \sqrt{2 \log \frac{A}{y}} \right),
	\end{align*}
	where $\phi$ is the p.d.f. of a normal standard distribution.
	\paragraph{$\mathbr{\tau < 0}$:} The derivation takes similar steps, all it takes is to call
	$\overline{\sigma}^2 = -\frac{1}{\tau\sigma_Q^2} = \frac{\sigma_P^2}{\sigma_P^2 - \sigma_Q^2}$, then the c.d.f. becomes:
	\begin{align*}
		F_{w}(y) &= \Phi\left(  \overline{\mu} + \sqrt{2\overline{\sigma}^2 \log \frac{y}{A}} \right)  -  \Phi\left(  \overline{\mu} - \sqrt{2\overline{\sigma}^2 \log \frac{y}{A}} \right),
	\end{align*}
	and the p.d.f. is:
	\begin{equation*}
		f_w(x) = \frac{\overline{\sigma}}{y\sqrt{\pi \log \frac{y}{A}} } \exp \left(  -\frac{1}{2} \overline{\mu}^2 \right) \left( \frac{A}{y} \right)^{\overline{\sigma}^2} \cosh \left( \overline{\mu} \overline{\sigma} \sqrt{2 \log \frac{y}{A}} \right).
	\end{equation*}
	To unify the two cases we set $\overline{\sigma}^2 = \frac{\sigma_P^2}{\left| \sigma_Q^2 - \sigma_P^2 \right|}$.
\end{proof} 

It is interesting to investigate the properties of the tail of the distribution when $w$ is unbounded. Indeed, we discover that the distribution displays a fat-tail behavior.
\begin{restatable}[]{prop}{fatTail}
	If $\sigma_P^2 > \sigma_Q^2$ then there exists $c >0$ and $y_0 > 0$ such that for any $y \ge y_0$, the p.d.f. $f_{w}$ can be lower bounded as $f_w(y) \ge c y ^{-1- \overline{\sigma}^2 }{(\log y)^{-\frac{1}{2}}}$.
\end{restatable}

\begin{proof}
	Let us call $z = y/A$ and let $a>0$ be a constant, then it holds that for sufficiently large $y$ we have:
	\begin{equation}
		f_w(y) \ge a z^{-1-\overline{\sigma}^2} (\log z) ^ {-1/2} \exp \left(\sqrt{\log z} \right)^{\sqrt{2} \overline{\mu} \overline{\sigma}}.
	\end{equation}
	To get the result, we observe that for $z>1$ we have $ \exp \left(\sqrt{\log z} \right) \ge 1$. Now, by replacing $z$ with $y/A$ we just need to change the constant $a$ into $c>0$.
\end{proof}

As a consequence, the $\alpha$-th moment of $w(x)$ does not exist for $\alpha -1-\overline{\sigma}^2 \ge  -1 \; \implies \;\alpha \ge \overline{\sigma}^2 = \frac{\sigma_P^2}{\sigma_P^2 - \sigma_Q^2}$, this prevents from using Bernstein-like inequalities for bounding in probability the importance weights. The non-existence of finite moments is confirmed by the $\alpha$-\Renyi divergence. Indeed, the $\alpha$-\Renyi divergence is defined when $\sigma_{\alpha}^2 = \alpha \sigma_Q^2 + (1-\alpha) \sigma_P^2 > 0$, \ie $\alpha < \frac{\sigma_P^2}{\sigma_P^2 - \sigma_Q^2}$.

\subsection{Equal variances} If $\sigma_Q^2 = \sigma_P^2 = \sigma^2$, the importance weights have the following expression:
\begin{equation}
	w(x) = \exp \left( \frac{\mu_P-\mu_Q}{\sigma^2} \left(x - \frac{\mu_P+\mu_Q}{2} \right) \right),
\end{equation}
for $x \sim Q$. The weight $w(x)$ is clearly unbounded and has $0$ as infimum value. Let us investigate its probability density function.

\begin{restatable}[]{prop}{}
 Let $Q \sim \mathcal{N}(\mu_Q, \sigma^2)$ be the behavioral distribution and $P \sim \mathcal{N}(\mu_P, \sigma^2)$ be the target distribution. The probability density function of $w(x) = q(x)/p(x)$ is given by:
\begin{equation}
	f_w(y) = \frac{\left| \widetilde{\sigma} \right|}{\sqrt{2\pi} y^{\frac{3}{2}}} \exp \left( -\frac{1}{2} \left( \widetilde{\mu}^2 +  \widetilde{\sigma}^2 \left( \log y \right)^2 \right) \right),
\end{equation}
where $\widetilde{\mu} = \frac{\mu_P -\mu_Q}{2\sigma}$ and $\widetilde{\sigma} = \frac{\sigma}{\mu_P - \mu_Q}$.
\end{restatable}

\begin{proof}
	We start computing the c.d.f.:
	\begin{align*}
		F_w(y) = \Pr\left( \exp \left\{ \frac{\mu_P-\mu_Q}{\sigma^2} \left(x - \frac{\mu_P+\mu_Q}{2} \right) \right\} \le y \right) = \Pr\left(  \frac{\mu_P-\mu_Q}{\sigma^2} \left(x - \frac{\mu_P+\mu_Q}{2} \right) \le \log y \right).
	\end{align*}
	First, we consider the case $\mu_P-\mu_Q > 0$ and observe that  $x = \mu_Q + \sigma z$, where $z\sim \mathcal{N}(0,1)$:
	\begin{align*}
		F_w(y) = \Pr\left( x  \le \frac{\mu_P+\mu_Q}{2} + \frac{\sigma^2}{\mu_P-\mu_Q}  \log y \right) = \Pr\left( z  \le \frac{\mu_P-\mu_Q}{2 \sigma}  + \frac{\sigma}{\mu_P-\mu_Q} \log y \right).
	\end{align*}
	We call  $\widetilde{\mu} = \frac{\mu_P -\mu_Q}{2\sigma}$ and $\widetilde{\sigma} = \frac{\sigma}{\mu_P - \mu_Q}$ and we have:
	\begin{align*}
		F_w(y) = \Pr\left( z  \le \widetilde{\mu} + \widetilde{\sigma}  \log y \right) = \Phi \left( \widetilde{\mu} + \widetilde{\sigma}  \log y \right).
	\end{align*}
	We take the derivative in order to get the density function:
	\begin{align*}
		f_{w}(y) &= \frac{\partial F_{w}(y)}{\partial y} = \frac{\widetilde{\sigma}}{y} \frac{1}{\sqrt{2\pi}} \exp \left( -\frac{1}{2} \left(\widetilde{\mu} + \widetilde{\sigma} \log y \right)^2 \right) = \frac{\widetilde{\sigma} }{\sqrt{2\pi} y^{\widetilde{\mu}\widetilde{\sigma}+1}} \exp \left( -\frac{1}{2} \left( \widetilde{\mu}^2 +  \widetilde{\sigma}^2 \left( \log y \right)^2 \right) \right).
	\end{align*}
	For the case $\mu_P-\mu_Q < 0$ the derivation is symmetric and the p.d.f. differs only by a minus sign. We account for this fact by considering $|\widetilde{\sigma}|$ in the final formula.
\end{proof}
In the case of equal variances, the tail behavior is different.
\begin{restatable}[]{prop}{thinTail}
	If $\sigma_P^2 = \sigma_Q^2$ then for any $\alpha >0$ there exist $c >0$ and $y_0 > 0$ such that for any $y \ge y_0$, the p.d.f. can be upper bounded as $f_w(y) \le c y^{-\alpha}$.
\end{restatable}

\begin{proof}
	Condensing all the constants in $c$, the p.d.f. can be written as:
	\begin{equation}
		f_{w}(y) = c y^{-3/2} \exp \left(\left( \log y \right)^2 \right) ^{-\frac{\widetilde{\sigma}^2}{2}}.
	\end{equation}
	For any $\alpha >0$, let us solve the following inequality:
	\begin{equation}
		y^{3/2} \exp \left(\left( \log y \right)^2 \right) ^{\frac{\widetilde{\sigma}^2}{2}} \ge y^{\alpha} \quad \implies \quad y \ge \exp \left( \frac{2}{\widetilde{\sigma}^2} \left(\alpha - \frac{3}{2} \right) \right).
	\end{equation}
	Thus, for $ y \ge \exp \left( \frac{2}{\widetilde{\sigma}^2} \left(\alpha - \frac{3}{2} \right) \right)$ we have that $f_w(y) \le c y^{-\alpha}$.
\end{proof}

This is sufficient to ensure the existence of the moments of any order, indeed the corresponding \Renyi divergence is: $\frac{\alpha(\mu_P-\mu_Q)^2}{2\sigma^2}$. By the way, the distribution of $w(x)$ remains subexponential, as $\exp \left(\left( \log y \right)^2 \right) ^{-\frac{\widetilde{\sigma}^2}{2}} \ge e^{-\eta y}$ for sufficiently large $y$.

Figure~\ref{fig:distr} reports the p.d.f. of the importance weights for different values of mean
and variance of the target distribution.

\begin{figure}
\centering
\begin{subfigure}{.5\textwidth}
  \centering
  \includegraphics[scale=1]{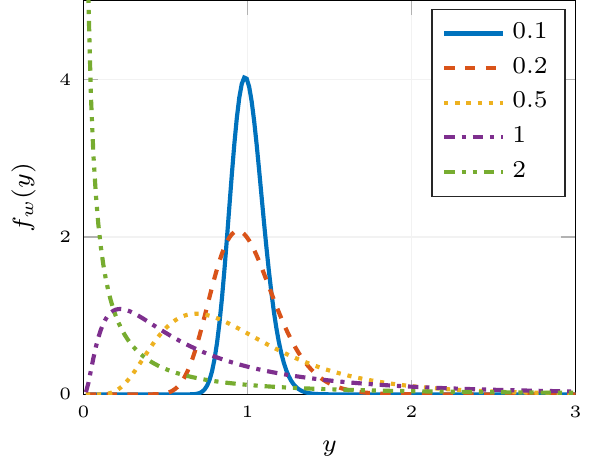}
  \label{fig:sub1}
  \caption{equal variance}
\end{subfigure}%
\begin{subfigure}{.5\textwidth}
  \centering
  \includegraphics[scale=1]{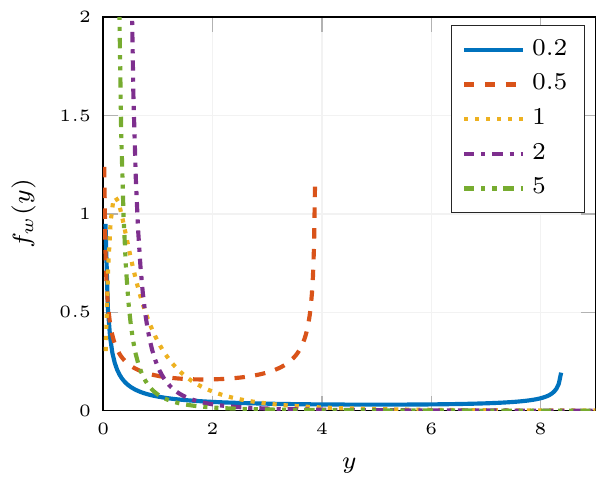}
  \label{fig:sub2}
   \caption{equal mean}
\end{subfigure}
\caption{Probability density function of the importance weights when the behavioral distribution is $\mathcal{N}(0,1)$ and the mean is changed keeping the variance equal to 1 (a) or the variance is changed keeping the target mean equal to 1 (b).}
\label{fig:distr}
\end{figure}

\section{Analysis of the SN Estimator}
\label{apx:SN}
In this Appendix, we provide some results regarding bias and variance of the self-normalized
importance sampling estimator. Let us start with the following result, derived from~\cite{cortes2010learning}, that bounds the expected squared difference
between non-self-normalized weight $w(x)$ and self-normalized weight $\widetilde{w}(x)$. 
\begin{restatable}[]{lemma}{}
\label{thr:lemmaBias}
	Let $P$ and $Q$ be two probability measures on the measurable space $\left(\mathcal{X}, \mathcal{F} \right)$ such that $P \ll Q$ and $d_2(P \| Q) < +\infty$. Let $x_1,x_2,\dots,x_N$ i.i.d. random variables sampled from $Q$. Then, for $N>0$ and for any $i=1,2,\dots,N$ it holds that:
	\begin{equation}
		\ev_{\mathbr{x} \sim Q} \left[ \left( \widetilde{w}_{P/Q}(x_i) - \frac{w_{P/Q}(x_i)}{N} \right)^2 \right] \le \frac{d_2( P \| Q) - 1}{N}.
	\end{equation}
\end{restatable}

\begin{proof}
	The result derives from simple algebraic manipulations and from the fact that $\Var_{x \sim Q} \left[ w_{P/Q}(x) \right] = d_2(P\| Q) -1$.
	\begin{align*}
		\ev_{\mathbr{x} \sim Q} & \left[ \left( \widetilde{w}_{P/Q}(x_i) - \frac{w_{P/Q}(x_i)}{N} \right)^2 \right]  = \ev_{\mathbr{x} \sim Q} \left[ \left( \frac{w_{P/Q}(x_i)}{\sum_{j=1}^{N} w_{P/Q}(x_j)} \right)^2 \left( 1 - \frac{\sum_{j=1}^N w_{P/Q}(x_j)}{N} \right)^2 \right]   \\
		& \le \ev_{\mathbr{x} \sim Q} \left[\left( 1 - \frac{\sum_{j=1}^N w_{P/Q}(x_j)}{N} \right)^2 \right] = \Var_{\mathbr{x} \sim Q} \left[\frac{\sum_{j=1}^N w_{P/Q}(x_j)}{N} \right] \\
		& = \frac{1}{N} \Var_{x_1 \sim Q} \left[ w_{P/Q}(x_1) \right] = \frac{d_2( P \| Q) - 1}{N}.
	\end{align*}
\end{proof}

A similar argument can be used to derive a bound on the bias of the SN estimator. 
\begin{restatable}[]{prop}{}
	Let $P$ and $Q$ be two probability measures on the measurable space $\left(\mathcal{X}, \mathcal{F} \right)$ such that $P \ll Q$ and $d_2(P \| Q) < +\infty$. Let $x_1,x_2,\dots,x_N$ i.i.d. random variables sampled from $Q$ and $f: \mathcal{X} \rightarrow \mathbb{R}$ be a bounded function ($\| f \|_{\infty}<\infty$). Then, the bias of the SN estimator can be bounded as:
	\begin{equation}
	\label{eq:boundBiasSN}
		\left| \ev_{\mathbr{x} \sim Q} \left[ \widetilde{\mu}_{P/Q} - \ev_{x \sim P} \left[ f(x) \right] \right] \right| \le \|f\|_{\infty} \min \left\{2,  \sqrt{\frac{d_2( P \| Q) - 1}{N}} \right\}.
	\end{equation}
\end{restatable}

\begin{proof}
	Since it holds that $|\widetilde{\mu}_{P/Q}| \le \|f\|_{\infty}$ the bias cannot be larger than $2 \|f\|_{\infty}$. We now derive a bound for the bias that vanishes as $N \rightarrow \infty$. We exploit the fact that the IS estimator is unbiased, \ie $\ev_{\mathbr{x} \sim Q} \left[ \widehat{\mu}_{P/Q} \right] = \ev_{x \sim P}  \left[ f(x) \right] $.
	\begin{align}
		\bigg| \ev_{\mathbr{x} \sim Q} & \left[  \widetilde{\mu}_{P/Q} - \ev_{x \sim P} \left[ f(x) \right] \right] \bigg|  = \left| \ev_{\mathbr{x} \sim Q} \left[ \widetilde{\mu}_{P/Q} - \ev_{\mathbr{x} \sim Q} \left[ \widehat{\mu}_{P/Q} \right] \right] \right| = \left| \ev_{\mathbr{x} \sim Q} \left[ \widetilde{\mu}_{P/Q} - \widehat{\mu}_{P/Q}  \right] \right|  \notag \\
		& \le  \ev_{\mathbr{x} \sim Q} \left[ \left| \widetilde{\mu}_{P/Q} - \widehat{\mu}_{P/Q}  \right| \right] = \notag \\
		& = \ev_{\mathbr{x} \sim Q} \left[ \left| \frac{\sum_{i=1}^N w_{P/Q}(x_i) f(x_i)}{\sum_{i=1}^N w_{P/Q}(x_i)}  - \frac{\sum_{i=1}^N w_{P/Q}(x_i) f(x_i)}{N}   \right| \right]  \notag \\
		& = \ev_{\mathbr{x} \sim Q} \left[ \left| \frac{\sum_{i=1}^N w_{P/Q}(x_i) f(x_i)}{\sum_{i=1}^N w_{P/Q}(x_i)} \right| \left| 1  - \frac{\sum_{i=1}^N w_{P/Q}(x_i)}{N}   \right| \right]  \label{line:1} \\
		& \le \ev_{\mathbr{x} \sim Q} \left[ \left( \frac{\sum_{i=1}^N w_{P/Q}(x_i) f(x_i)}{\sum_{i=1}^N w_{P/Q}(x_i)} \right)^2 \right]^{\frac{1}{2}} \ev_{\mathbr{x} \sim Q} \left[ \left( 1  - \frac{\sum_{i=1}^N w_{P/Q}(x_i)}{N} \right)^2 \right]^{\frac{1}{2}}  \label{line:2} \\
		& \le \| f \|_{\infty} \sqrt{\frac{d_2( P \| Q) - 1}{N}} \label{line:3},
	\end{align}
	where \eqref{line:2} follows from \eqref{line:1} by applying Cauchy-Schwartz inequality and \eqref{line:3} is obtained by observing that $\left( \frac{\sum_{i=1}^N w_{P/Q}(x_i) f(x_i)}{\sum_{i=1}^N w_{P/Q}(x_i)} \right)^2 \le \|f \|_{\infty}^2$. 
\end{proof}

Bounding the variance of the SN estimator is non-trivial since the the normalization term makes all the samples interdependent. Exploiting the boundedness of $\widetilde{\mu}_{P/Q}$ we can derive trivial bounds like: $\Var_{\mathbr{x} \sim Q} \left[ \widetilde{\mu}_{P/Q} \right] \le \|f\|_{\infty}^2$. However, this bound does not shrink with the number of samples $N$. Several approximations of the variance have been proposed, like the following derived using the delta method~\cite{ver2012invented, mcbook}:
\begin{equation}
\label{eq:approxVar}
	\Var_{\mathbr{x} \sim Q} \left[ \widetilde{\mu}_{P/Q} \right] = \frac{1}{N} \ev_{x_1 \sim Q} \left[  w_{P/Q}^2(x_1) \left( f(x_1) - \ev_{x \sim P} \left[ f(x) \right] \right)^2\right] + o(N^{-2}). 
\end{equation}

We will not use the approximate expression for the variance, but we will directly bound the Mean Squared Error (MSE) of the SN estimator, which
is the sum of the variance and the bias squared. 
\begin{restatable}[]{prop}{}
	Let $P$ and $Q$ be two probability measures on the measurable space $\left(\mathcal{X}, \mathcal{F} \right)$ such that $P \ll Q$ and $d_2(P \| Q) < +\infty$. Let $x_1,x_2,\dots,x_N$ i.i.d. random variables sampled from $Q$ and $f: \mathcal{X} \rightarrow \mathbb{R}$ be a bounded function ($\| f \|_{\infty}<+\infty$). Then, the $\mathrm{MSE}$ of the SN estimator can be bounded as:
	\begin{equation}
	\label{eq:boundMSE}
		\mathrm{MSE}_{\mathbr{x} \sim Q} \left[ \widetilde{\mu}_{P/Q} \right] \le 2 \|f\|_{\infty}^2 \min \left\{2,  \frac{2d_2( P \| Q)-1}{N} \right\}.
	\end{equation}
\end{restatable}

\begin{proof}
	First, recall that $\widetilde{\mu}_{P/Q}$ is bounded by $\| f\|_{\infty}$ thus its MSE cannot be larger than $4\| f\|_{\infty}^2$. The idea
	of the proof is to sum and subtract the IS estimator $\widehat{\mu}_{P/Q}$:
	\begin{align}
		\mathrm{MSE}_{\mathbr{x} \sim Q} \left[ \widetilde{\mu}_{P/Q} \right] & = \ev_{\mathbr{x} \sim Q} \left[ \left( \widetilde{\mu}_{P/Q} - \ev_{x \sim P} \left[ f(x) \right] \right)^2 \right] \notag \\
		& = \ev_{\mathbr{x} \sim Q} \left[ \left( \widetilde{\mu}_{P/Q} - \ev_{x \sim P} \left[ f(x) \right]  \pm \widehat{\mu}_{P/Q} \right)^2 \right] \label{line:10} \\
		& \le 2 \ev_{\mathbr{x} \sim Q} \left[ \left( \widetilde{\mu}_{P/Q} - \widehat{\mu}_{P/Q} \right)^2 \right] + 2\ev_{\mathbr{x} \sim Q} \left[ \left( \widehat{\mu}_{P/Q} -  \ev_{x \sim P}[f(x)]\right)^2 \right]  \label{line:11} \\
		& \le 2 \ev_{\mathbr{x} \sim Q} \left[ \left( \frac{\sum_{i=1}^N w_{P/Q}(x_i) f(x_i)}{\sum_{i=1}^N w_{P/Q}(x_i)} \right)^2 \left( 1 -  \frac{\sum_{i=1}^N w_{P/Q}(x_i)}{N} \right)^2 \right] + 2 \Var_{\mathbr{x} \sim Q} \left[  \widehat{\mu}_{P/Q} \right] \label{line:12} \\
		& \le 2 \|f \|_{\infty}^2 \ev_{\mathbr{x} \sim Q} \left[ \left( 1 -  \frac{\sum_{i=1}^N w_{P/Q}(x_i)}{N} \right)^2 \right] + 2 \Var_{\mathbr{x} \sim Q} \left[  \widehat{\mu}_{P/Q} \right] \label{line:13} \\
		& \le 2 \|f \|_{\infty}^2 \Var_{\mathbr{x} \sim Q} \left[ \frac{\sum_{i=1}^N w_{P/Q}(x_i)}{N} \right] + 2 \Var_{\mathbr{x} \sim Q} \left[  \widehat{\mu}_{P/Q} \right] \label{line:14}\\
		& \le 2 \|f \|_{\infty}^2 \frac{d_2(P \| Q) -1 }{N} + 2 \|f \|_{\infty}^2 \frac{d_2(P\| Q)}{N} =2 \|f \|_{\infty}^2 \frac{2d_2( P \| Q) - 1}{N}\label{line:14},
	\end{align}
	where line \eqref{line:11} follows from line \eqref{line:10} by applying the inequality $(a+b)^2 \le 2(a^2+b^2)$, \eqref{line:13} follows from \eqref{line:12} by observing that $\left( \frac{\sum_{i=1}^N w_{P/Q}(x_i) f(x_i)}{\sum_{i=1}^N w_{P/Q}(x_i)} \right)^2 \le \|f \|_{\infty}^2$.
\end{proof}

We can use this result to provide a high confidence bound for the SN estimator.
\begin{restatable}[]{prop}{}
	Let $P$ and $Q$ be two probability measures on the measurable space $\left(\mathcal{X}, \mathcal{F} \right)$ such that $P \ll Q$ and $d_2(P \| Q) < +\infty$. Let $x_1,x_2,\dots,x_N$ i.i.d. random variables sampled from $Q$ and $f: \mathcal{X} \rightarrow \mathbb{R}$ be a bounded function ($\| f \|_{\infty}<+\infty$). Then, for any $0< \delta \le 1$ and $N>0$ with probability at least $1-\delta$:
	\begin{equation*}
	\label{eq:boundSN}
		\ev_{x \sim P} \left[ f(x) \right] \ge \frac{1}{N} \sum_{i=1}^N \widetilde{w}_{P/Q}(x_i) f(x_i) -  2\| f \|_{\infty} \min \left\{ 1, \sqrt{\frac{d_2(P\|Q) (4-3\delta)}{\delta N}} \right\}.
	\end{equation*}
\end{restatable}
\begin{proof}
	The result is obtained by applying Cantelli's inequality and accounting for the bias. Consider the random variable $\widetilde{\mu}_{P/Q} = \frac{1}{N} \sum_{i=1}^N \widetilde{w}_{P/Q}(x_i) f(x_i)$ and let $\widetilde{\lambda} = \lambda - \left| \ev_{x \sim P}\left[ f(x) \right] - \ev_{\mathbr{x} \sim P}\left[ \widetilde{\mu}_{P/Q} \right] \right|$:
	\begin{align*}
		\Pr \left( \widetilde{\mu}_{P/Q} -\ev_{x \sim P}\left[ f(x) \right]  \ge \lambda \right) & = \Pr \left( \widetilde{\mu}_{P/Q} - \ev_{\mathbr{x} \sim P}\left[ \widetilde{\mu}_{P/Q} \right] \ge \lambda + \ev_{x \sim P}\left[ f(x) \right] - \ev_{\mathbr{x} \sim P}\left[ \widetilde{\mu}_{P/Q}  \right] \right) \\
		& \le \Pr \left( \widetilde{\mu}_{P/Q} - \ev_{\mathbr{x} \sim P}\left[ \widetilde{\mu}_{P/Q} \right] \ge \lambda - \left| \ev_{x \sim P}\left[ f(x) \right] - \ev_{\mathbr{x} \sim P}\left[ \widetilde{\mu}_{P/Q}  \right] \right| \right) \\
		& = \Pr \left( \widetilde{\mu}_{P/Q} - \ev_{\mathbr{x} \sim P}\left[ \widetilde{\mu}_{P/Q} \right] \ge \widetilde{\lambda}  \right).
	\end{align*}
	Now we apply Cantelli's inequality:
	\begin{align}
		\Pr \left( \widetilde{\mu}_{P/Q} -\ev_{x \sim P} \left[ f(x) \right] \ge \lambda \right) &\le 
		\Pr \left( \widetilde{\mu}_{P/Q} -\ev_{x \sim P} \left[ \widetilde{\mu}_{P/Q} \right] \ge \widetilde{\lambda} \right) \le
		 \frac{1}{1+\frac{\widetilde{\lambda}^2}{\Var_{\mathbr{x} \sim Q} \left[ \widetilde{\mu}_{P/Q} \right]}} \nonumber \\
		 &=  \frac{1}{1+\frac{\left(\lambda -  \left| \ev_{x \sim P}\left[ f(x) \right] - \ev_{\mathbr{x} \sim P}\left[ \widetilde{\mu}_{P/Q} \right] \right| \right)^2}{\Var_{\mathbr{x} \sim Q} \left[ \widetilde{\mu}_{P/Q} \right]}} .
	\end{align}
	By calling $\delta = \frac{1}{1+\frac{\left(\lambda -  \left| \ev_{x \sim P}\left[ f(x) \right] - \ev_{\mathbr{x} \sim P}\left[ \widetilde{\mu}_{P/Q} \right] \right| \right)^2}{\Var_{\mathbr{x} \sim Q} \left[ \widetilde{\mu}_{P/Q} \right]}}$ and considering the complementary event, we get that with probability at least $1-\delta$ we have:
	\begin{equation}
		\ev_{x \sim P} \left[ f(x) \right] \ge \widetilde{\mu}_{P/Q} - \left| \ev_{x \sim P}\left[ f(x) \right] - \ev_{\mathbr{x} \sim P}\left[ \widetilde{\mu}_{P/Q} \right] \right| - \sqrt{\frac{1-\delta}{\delta} \Var_{\mathbr{x} \sim Q} \left[ \widetilde{\mu}_{P/Q} \right]}
	\end{equation}
	Then we bound the bias term $\left| \ev_{x \sim P}\left[ f(x) \right] - \ev_{\mathbr{x} \sim P}\left[ \widetilde{\mu}_{P/Q} \right] \right|$ with equation~\eqref{eq:boundBiasSN} and the variance term with the MSE in equation~\eqref{eq:boundMSE}. With some simple algebraic manipulation we have:
	\begin{align*}
		\ev_{x \sim P} \left[ f(x) \right] & \ge \widetilde{\mu}_{P/Q} - \|f\|_{\infty} \sqrt{\frac{d_2(P\|Q) - 1}{N}} - \|f\|_{\infty} \sqrt{\frac{1-\delta}{\delta}  \frac{2(2d_2(P\|Q) - 1)}{N}} \\
		& \ge \widetilde{\mu}_{P/Q} - \|f\|_{\infty} \sqrt{\frac{d_2(P\|Q)}{N}} - \|f\|_{\infty} \sqrt{\frac{1-\delta}{\delta}  \frac{4d_2(P\|Q)}{N}} \\
		& = \widetilde{\mu}_{P/Q} - \|f\|_{\infty} \sqrt{\frac{d_2(P\|Q)}{N}} \left( 1 + 2 \sqrt{\frac{1-\delta}{\delta}} \right) \\
		& \ge \widetilde{\mu}_{P/Q} - 2 \|f\|_{\infty} \sqrt{\frac{d_2(P\|Q)}{N}}  \sqrt{ 1 +  \frac{4(1-\delta)}{\delta}} \\
		& \ge \widetilde{\mu}_{P/Q} - 2 \|f\|_{\infty} \sqrt{\frac{d_2(P\|Q) (4-3\delta)}{\delta N}},
\end{align*}	 
	where the last line follows from the fact that $\sqrt{a} + \sqrt{b} \le 2 \sqrt{a+b}$ for any $a,b\ge 0$. Finally, recalling that the range of the SN estimator is $2\|f\|_{\infty}$ we get the result.
\end{proof}

It is worth noting that, apart for the constants, the bound has the same dependence on $d_2$ as in Theorem~\ref{thr:bound}. Thus, by suitably redefining the hyperparameter $\lambda$ we can optimize the same surrogate objective function for both IS and SN estimators.

\section{Implementation details}
In this Appendix, we provide some aspects about our implementation of POIS.
\label{apx:impl}

\subsection{Line Search}
\label{apx:lineSearch}
%
At each offline iteration $k$ the parameter update is performed in the direction of ${\mathcal{G}}(\mathbr{\theta}^{j}_{k})^{-1} {\nabla}_{\mathbr{\theta}_{k}^{j}}{\mathcal{L}}(\mathbr{\theta}^{j}_{k} / \mathbr{\theta}^{j}_{0})$ with a step size $\alpha_k$ determined in order to maximize the improvement. For brevity we will remove subscripts and dependence on $\mathbr{\theta}^j_0$ from the involved quantities. The rationale behind our line search is the following. Suppose that our
objective function $\mathcal{L}(\mathbr{\theta})$, restricted to the gradient direction $\mathcal{G}^{-1}(\mathbr{\theta}) \nabla_{\mathbr{\theta}} \mathcal{L}(\mathbr{\theta})$, represents a concave parabola in the Riemann manifold having $\mathcal{G}(\mathbr{\theta})$ as Riemann metric tensor. Suppose we know a point $\mathbr{\theta}_0$, the Riemann gradient in that point $ \mathcal{G}(\mathbr{\theta}_0)^{-1} \nabla_{\mathbr{\theta}} \mathcal{L}(\mathbr{\theta}_0)$ and another point:  $\mathbr{\theta}_{l} = \mathbr{\theta}_0 + \alpha_l \mathcal{G}(\mathbr{\theta}_0)^{-1} \nabla_{\mathbr{\theta}} \mathcal{L}(\mathbr{\theta}_0)$. For both points we know the value of the loss function: $\mathcal{L}_0 = \mathcal{L}(\mathbr{\theta}_0)$ and $\mathcal{L}_{l} = \mathcal{L}(\mathbr{\theta}_{l})$ and indicate with $\Delta \mathcal{L}_l = \mathcal{L}_{l} - \mathcal{L}_0$ the objective function improvement. Having this information we can compute the vertex of that parabola, which is its global maximum. 
Let us call $l(\alpha) = \mathcal{L}\left(\mathbr{\theta}_0 + \alpha \mathcal{G}^{-1}(\mathbr{\theta}_0) \nabla_{\mathbr{\theta}} \mathcal{L}(\mathbr{\theta}_0) \right) - \mathcal{L}(\mathbr{\theta}_0)$, being a parabola it can be expressed as $l(\alpha) = a\alpha^2 + b\alpha + c$. Clearly, $c=0$ by definition of $l(\alpha)$; $a$ and $b$ can be determined by enforcing the conditions: 
\begin{align*}
	b = \frac{\partial l}{ \partial \alpha} \bigg\rvert_{\alpha=0} & = \frac{\partial}{ \partial \alpha} \mathcal{L}\left(\mathbr{\theta}_0 + \alpha \mathcal{G}^{-1}(\mathbr{\theta}_0) \nabla_{\mathbr{\theta}} \mathcal{L}(\mathbr{\theta}_0) \right) - \mathcal{L}(\mathbr{\theta}_0)\rvert_{\alpha=0} = \\
	& = \nabla_{\mathbr{\theta}} \mathcal{L}(\mathbr{\theta}_0)^{T} \mathcal{G}^{-1}(\mathbr{\theta}_0) \nabla_{\mathbr{\theta}} \mathcal{L}(\mathbr{\theta}_0) = \\
	& = \| \nabla_{\mathbr{\theta}} \mathcal{L}(\mathbr{\theta}_0) \|_{\mathcal{G}^{-1}(\mathbr{\theta}_0)}^{2}, 
\end{align*}
\begin{align*}
	l(\alpha_l) & = a \alpha_l^2 + b \alpha_l = a \alpha_l^2 + \| \nabla_{\mathbr{\theta}} \mathcal{L}(\mathbr{\theta}_0) \|_{\mathcal{G}^{-1}(\mathbr{\theta}_0)}^{2} \alpha_l = \Delta \mathcal{L}_l \quad \implies \\
	& \implies \quad a = \frac{\Delta \mathcal{L}_l - \| \nabla_{\mathbr{\theta}} \mathcal{L}(\mathbr{\theta}_0) \|_{\mathcal{G}^{-1}(\mathbr{\theta}_0)}^{2} \alpha_l}{\alpha_l^2}.
\end{align*}
Therefore, the parabola has the form:
\begin{equation}
	l(\alpha) = \frac{\Delta \mathcal{L}_l - \| \nabla_{\mathbr{\theta}} \mathcal{L}(\mathbr{\theta}_0) \|_{\mathcal{G}^{-1}(\mathbr{\theta}_0)}^{2} \alpha_l}{\alpha_l^2} \alpha^2 + \| \nabla_{\mathbr{\theta}} \mathcal{L}(\mathbr{\theta}_0) \|_{\mathcal{G}^{-1}(\mathbr{\theta}_0)}^{2} \alpha.
\end{equation}
Clearly, the parabola is concave only if $\Delta \mathcal{L}_l < \| \nabla_{\mathbr{\theta}} \mathcal{L}(\mathbr{\theta}_0) \|_{\mathcal{G}^{-1}(\mathbr{\theta}_0)}^{2} \alpha_l$. The vertex is located at:
\begin{equation}
	\alpha_{l+1} = \frac{\| \nabla_{\mathbr{\theta}} \mathcal{L}(\mathbr{\theta}_0) \|_{\mathcal{G}^{-1}(\mathbr{\theta}_0)}^{2} \alpha_l^2 }{2\left(\| \nabla_{\mathbr{\theta}} \mathcal{L}(\mathbr{\theta}_0) \|_{\mathcal{G}^{-1}(\mathbr{\theta}_0)}^{2} \alpha_l - \Delta \mathcal{L}_l\right)}.
\end{equation}
To simplify the expression, like in~\cite{matsubara2010adaptive} we define $\alpha_l = \epsilon_l /  \| \nabla_{\mathbr{\theta}} \mathcal{L}(\mathbr{\theta}_0) \|_{\mathcal{G}^{-1}(\mathbr{\theta}_0)}^2$. Thus, we get:
\begin{equation}
\label{eq:epsStar}
	\epsilon_{l+1} = \frac{\epsilon_l^2 }{2(\epsilon_l - \Delta \mathcal{L}_l)}.
\end{equation}
Of course, we need also to manage the case in which the parabola is convex, \ie $\Delta \mathcal{L}_l \ge \| \nabla_{\mathbr{\theta}} \mathcal{L}(\mathbr{\theta}_0) \|_{\mathcal{G}^{-1}(\mathbr{\theta}_0)}^{2} \alpha_l$. Since our objective function is not really a parabola we reinterpret the two
cases: i) $\Delta \mathcal{L}_l > \| \nabla_{\mathbr{\theta}} \mathcal{L}(\mathbr{\theta}_0) \|_{\mathcal{G}^{-1}(\mathbr{\theta}_0)}^{2} \alpha_l$, the function is sublinear and in this case we use \eqref{eq:epsStar} to determine the new step size $\alpha_{l+1} = \epsilon_{l+1} /  \| \nabla_{\mathbr{\theta}} \mathcal{L}(\mathbr{\theta}_0) \|_{\mathcal{G}^{-1}(\mathbr{\theta}_0)}^2$; ii) $\Delta \mathcal{L}_l \ge \| \nabla_{\mathbr{\theta}} \mathcal{L}(\mathbr{\theta}_0) \|_{\mathcal{G}^{-1}(\mathbr{\theta}_0)}^{2} \alpha_l$, the function is superlinear, in this
case we increase the step size multiplying by $\eta>1$, \ie $\alpha_{l+1} = \eta \alpha_l$. Finally the update rule becomes:
\begin{equation}
	\epsilon_{l+1} = \begin{cases}
						\eta \epsilon_l & \text{if } \Delta \mathcal{L}_l > \frac{\epsilon_l (2\eta - 1)}{2 \eta} \\
						\frac{\epsilon_l^2}{2 (\epsilon_l - \Delta \mathcal{L}_l )} & \text{otherwise}
					 \end{cases}.
\end{equation}
The procedure is iterated until a maximum number of attempts is reached (say 30) or the objective function improvement is too small (say 1e-4).
The pseudocode of the line search is reported in Algorithm~\ref{alg:lineSearch}.
\begin{algorithm}[H]
  \caption{Parabolic Line Search}
  \label{alg:lineSearch}
   \hspace*{\algorithmicindent} \textbf{Input}: $\mathrm{tol}_{\Delta\mathcal{L}} = 1e-4$, $M_{\mathrm{ls}}=30$, $\mathcal{L}_0$ \\
   \hspace*{\algorithmicindent} \textbf{Output} : $\alpha^*$
    \begin{algorithmic}
    \State $\alpha_0 = 0$
	\State $\epsilon_1 = 1$
	\State $\Delta \mathcal{L}_{k-1} = - \infty$
	\For{$l=1,2,\dots, M_{\mathrm{ls}}$}
		\State $\alpha_l = \epsilon_l / \| \nabla_{\mathbr{\theta}} \mathcal{L}(\mathbr{\theta}_0) \|_{\mathcal{G}^{-1}(\mathbr{\theta}_0)}^2$
		\State $\mathbr{\theta}_l = \alpha_l \mathcal{G}^{-1}(\mathbr{\theta}_0)  \nabla_{\mathbr{\theta}} \mathcal{L}(\mathbr{\theta}_0)$
		\State $\Delta \mathcal{L}_l = \mathcal{L}_l - \mathcal{L}_0$
		\If{$\Delta \mathcal{L}_l  < \Delta \mathcal{L}_{l-1} + \mathrm{tol}_{\Delta\mathcal{L}} $}
			\State \Return $\alpha_{l-1}$
		\EndIf
		\State {$
			 \epsilon_{l+1} = \begin{cases}
						\eta \epsilon_l & \text{if } \Delta \mathcal{L}_l > \frac{\epsilon_l (1-2\eta)}{2\eta}\\
						\frac{\epsilon_l^2}{2 (\epsilon_l - \Delta \mathcal{L}_l )} & \text{otherwise}
					 \end{cases}
		$}
	\EndFor
\end{algorithmic}
  \end{algorithm}

\subsection{Estimation of the \Renyi divergence}
\label{apx:estRenyi}
In A-POIS, the \Renyi divergence needs to be computed between the behavioral $p(\cdot|\mathbr{\theta})$ and target $p(\cdot|\mathbr{\theta}')$ distributions on trajectories. This is likely impractical as it requires to integrate over the trajectory space. Moreover, for stochastic environments it cannot be computed unless we know the transition model $P$. The following result provides an exact, although loose, bound to this quantity in the case of finite-horizon tasks.
\begin{restatable}[]{prop}{}
	Let $p(\cdot|\mathbr{\theta})$ and $p(\cdot|\mathbr{\theta}')$ be the behavioral and target
	trajectory probability density functions. Let $H< \infty$ be the task-horizon. Then, it holds that:
	\begin{equation*}
		d_{\alpha}\left( p(\cdot|\mathbr{\theta}') \| p(\cdot|\mathbr{\theta}) \right) \le \left( \sup_{s \in \mathcal{S}} d_{\alpha}\left( \pi_{\mathbr{\theta}'}(\cdot|s) \|  \pi_{\mathbr{\theta}}(\cdot|s)  \right) \right)^H.
	\end{equation*}
\end{restatable}

\begin{proof}
	We prove the proposition by induction on the horizon $H$. We define $d_{\alpha, H}$ as the
	$\alpha$-\Renyi divergence at horizon $H$. For $H=1$ we have:
	\begin{align*}
		d_{\alpha, 1}\left( p(\cdot|\mathbr{\theta}') \| p(\cdot|\mathbr{\theta}) \right) & = 
			\int_{\mathcal{S}} D(s_0) \int_{\mathcal{A}} \pi_{\mathbr{\theta}}(a_0|s_0) \left(\frac{ \pi_{\mathbr{\theta}'}(a_0|s_0)}{ \pi_{\mathbr{\theta}}(a_0|s_0)} \right)^{\alpha} \int_{\mathcal{S}} P(s_1|s_0,a_0) \de s_1 \de a_0 \de s_0 \\
			&  = \int_{\mathcal{S}} D(s_0) \int_{\mathcal{A}} \pi_{\mathbr{\theta}}(a_0|s_0) \left(\frac{ \pi_{\mathbr{\theta}'}(a_0|s_0)}{ \pi_{\mathbr{\theta}}(a_0|s_0)} \right)^{\alpha}  \de a_0 \de s_0 \\
			& \le \int_{\mathcal{S}} D(s_0) \de s_0 \sup_{s \in \mathcal{S}} \int_{\mathcal{A}} \pi_{\mathbr{\theta}}(a_0|s) \left(\frac{ \pi_{\mathbr{\theta}'}(a_0|s)}{ \pi_{\mathbr{\theta}}(a_0|s)} \right)^{\alpha}  \de a_0 \\
			& \le \sup_{s \in \mathcal{S}} d_{\alpha}\left( \pi_{\mathbr{\theta}'}(\cdot|s) \|  \pi_{\mathbr{\theta}}(\cdot|s) \right),
	\end{align*}
	where the last but one passage follows from Holder's inequality. Suppose that the proposition
	holds for any $H' < H$, let us prove the proposition for $H$.
	\begin{align*}
		d_{\alpha, H}\Big( p(\cdot|\mathbr{\theta}') & \| p(\cdot|\mathbr{\theta}) \Big) = 
			\int_{\mathcal{S}} D(s_0) \; \dots \int_{\mathcal{A}} \pi_{\mathbr{\theta}}(a_{H-2}|s_{H-2}) \left(\frac{ \pi_{\mathbr{\theta}'}(a_{H-2}|s_{H-2})}{ \pi_{\mathbr{\theta}}(a_{H-2}|s_{H-2})} \right)^{\alpha} \int_{\mathcal{S}} P(s_{H-1}|s_{H-2},a_{H-2}) \\
			& \quad \times \int_{\mathcal{A}} \pi_{\mathbr{\theta}}(a_{H-1}|s_{H-1}) \left(\frac{ \pi_{\mathbr{\theta}'}(a_{H-1}|s_{H-1}) }{ \pi_{\mathbr{\theta}}(a_{H-1}|s_{H-1}) } \right)^{\alpha} \int_{\mathcal{S}} P(s_H|s_{H-1},a_{H-1}) \de s_0 \dots \de s_{H-1} \\
			& \quad \times \de a_{H-2} \de s_{H-1} \de a_{H-1} \de s_H \\
			& = \int_{\mathcal{S}} D(s_0) \; \dots \int_{\mathcal{A}} \pi_{\mathbr{\theta}}(a_{H-2}|s_{H-2}) \left(\frac{ \pi_{\mathbr{\theta}'}(a_{H-2}|s_{H-2})}{ \pi_{\mathbr{\theta}}(a_{H-2}|s_{H-2})} \right)^{\alpha} \int_{\mathcal{S}} P(s_{H-1}|s_{H-2},a_{H-2})  \\
			& \quad \times \int_{\mathcal{A}} \pi_{\mathbr{\theta}}(a_{H-1}|s_{H-1}) \left(\frac{ \pi_{\mathbr{\theta}'}(a_{H-1}|s_{H-1}) }{ \pi_{\mathbr{\theta}}(a_{H-1}|s_{H-1}) } \right)^{\alpha} \de s_0 \dots \de s_{H-1} \de a_{H-2} \de s_{H-1} \de a_{H-1}  \\
			& \le \int_{\mathcal{S}} D(s_0) \; \dots \int_{\mathcal{A}} \pi_{\mathbr{\theta}}(a_{H-2}|s_{H-2}) \left(\frac{ \pi_{\mathbr{\theta}'}(a_{H-2}|s_{H-2})}{ \pi_{\mathbr{\theta}}(a_{H-2}|s_{H-2})} \right)^{\alpha} \int_{\mathcal{S}} P(s_{H-1}|s_{H-2},a_{H-2}) \\
			& \quad \times  \de s_0 \dots \de s_{H-1} \de a_{H-2} \de s_{H-1} \times \sup_{s \in \mathcal{S}} \int_{\mathcal{A}} \pi_{\mathbr{\theta}}(a_{H-1}|s) \left(\frac{ \pi_{\mathbr{\theta}'}(a_{H-1}|s) }{ \pi_{\mathbr{\theta}}(a_{H-1}|s) } \right)^{\alpha}\de a_{H-1} \ \\
			& \le d_{\alpha, H-1}\left( p(\cdot|\mathbr{\theta}') \| p(\cdot|\mathbr{\theta}) \right) \sup_{s \in \mathcal{S}} d_{\alpha}\left( \pi_{\mathbr{\theta}'}(\cdot|s) \|  \pi_{\mathbr{\theta}}(\cdot|s) \right) \ \\
		& \le \left( \sup_{s \in \mathcal{S}} d_{\alpha}\left( \pi_{\mathbr{\theta}'}(\cdot|s) \|  \pi_{\mathbr{\theta}}(\cdot|s) \right) \right)^H,
	\end{align*}
%
	where we applied Holder's inequality again and the last passage is obtained for the inductive hypothesis.
\end{proof}

The proposed bound, however, is typically ultraconservative, thus we propose two alternative estimators of the $\alpha$-\Renyi divergence. The first estimator is obtained by simply
rephrasing the definition~\eqref{eq:renyiDiv} into a sample-based version:
\begin{equation}
	\widehat{d}_\alpha \left( P \| Q \right) = \frac{1}{N} \sum_{i=1}^N \left(\frac{p(x_i)}{q(x_i)}\right)^{\alpha} = \frac{1}{N} \sum_{i=1}^N w_{P/Q}^{\alpha}(x_i),
\end{equation}
where $x_i \sim Q$. This estimator is clearly unbiased and applies to any pair of probability distributions. However, in A-POIS
$P$ and $Q$ are distributions over trajectories, their densities are expressed as products, thus the $\alpha$-\Renyi divergence becomes:
\begin{align*}
	d_{\alpha}\left( p(\cdot|\mathbr{\theta}') \| p(\cdot|\mathbr{\theta}) \right)  & = \int_{\mathcal{T}} p(\cdot|\mathbr{\theta})(\tau) \left( \frac{p(\tau|\mathbr{\theta}')}{p(\tau|\mathbr{\theta})} \right)^{\alpha} \de \tau = \\
	& = \int_{\mathcal{T}} D(s_{\tau,0}) \prod_{t=0}^{H-1}P(s_{\tau,t+1}|s_{\tau,t}, a_{\tau,t}) \prod_{t=0}^{H-1} \pi_{\mathbr{\theta}}(a_{\tau,t} | s_{\tau,t}) \left( \frac{ \pi_{\mathbr{\theta}'}(a_{\tau,t} | s_{\tau,t}) }{ \pi_{\mathbr{\theta}}(a_{\tau,t} | s_{\tau,t})} \right)^{\alpha}  \de \tau.
\end{align*}
Since both $ \pi_{\mathbr{\theta}}$ and $ \pi_{\mathbr{\theta}'}$ are known we are able to compute exactly for each state $d_{\alpha} \left( \pi_{\mathbr{\theta}'}(\cdot|s) \| \pi_{\mathbr{\theta}}(\cdot|s) \right)$ with no need to sample the action $a$. Therefore, we suggest to
estimate the \Renyi divergence between two trajectory distributions as:
\begin{equation}
	\widehat{d}_{\alpha}\left( p(\cdot|\mathbr{\theta}') \| p(\cdot|\mathbr{\theta}) \right) = \frac{1}{N} \sum_{i=1}^N \prod_{t=0}^{H-1} d_{\alpha} \left( \pi_{\mathbr{\theta}'}(\cdot|s_{\tau_i,t}) \| \pi_{\mathbr{\theta}}(\cdot|s_{\tau_i,t}) \right).
\end{equation}

\subsection{Computation of the Fisher Matrix}
\label{apx:implFIM}
In A-POIS the Fisher Information Matrix needs to be estimated off-policy from samples. We can
use, for this purpose, the IS estimator:
\begin{equation*}
	\widehat{\mathcal{F}}(\mathbr{\theta}'/\mathbr{\theta}) = \frac{1}{N} \sum_{i=1}^N w_{\mathbr{\theta}'/\mathbr{\theta}}(\tau_i) \left( \sum_{t=0}^{H-1} \nabla_{\mathbr{\theta}'} \log \pi_{\mathbr{\theta}'}(a_{\tau_i,t}|s_{\tau_i,t})\right)^T \left( \sum_{t=0}^{H-1} \nabla_{\mathbr{\theta}'} \log \pi_{\mathbr{\theta}'}(a_{\tau_i,t}|s_{\tau_i,t})\right).
\end{equation*}
The SN estimator is obtained by replacing $w_{\mathbr{\theta}'/\mathbr{\theta}}(\tau_i)$ with $\widetilde{w}_{\mathbr{\theta}'/\mathbr{\theta}}(\tau_i)$. Those estimators become very unreliable
when $\mathbr{\theta}'$ is far from $\mathbr{\theta}$, making them difficult to use in practice. On the contrary, in P-POIS in presence of Gaussian hyperpolicies the FIM can be computed exactly~\cite{sun2009efficient}. If the hyperpolicy has diagonal covariance matrix, \ie $\nu_{\mathbr{\mu}, \mathbr{\sigma}} = \mathcal{N}(\mathbr{\mu}, \mathrm{diag} (\mathbr{\sigma}^2))$, the FIM is also diagonal:
\begin{equation*}
	\mathcal{F}(\mathbr{\mu}, \mathbr{\sigma}) = \left(
\begin{array}{c|c}
\mathrm{diag}(1 / \mathbr{\sigma}^2) & \mathbr{0} \\
\hline
\mathbr{0} & 2 \mathbr{I}
\end{array}
\right),
\end{equation*}
where $\mathbr{I}$ is a properly-sized identity matrix.

\subsection{Practical surrogate objective functions}\label{apx:psof}
In practice, the \Renyi divergence term $d_2$ in the surrogate objective functions presented so far, either exact in P-POIS or approximate in A-POIS, tends
to be overly-conservative. To mitigate this problem, by observing that $d_2(P \| Q) / N = 1/ \ess(P \| Q)$ from equation~\eqref{eq:ess} we can replace the whole quantity with an estimator like $\widehat{\ess}(P \| Q)$, as presented in equation~\eqref{eq:ess}. This leads to the following approximated surrogate objective functions:
\begin{equation*}
	\widetilde{\mathcal{L}}_{\lambda}^{\mathrm{A-POIS}}(\mathbr{\theta'}/\mathbr{\theta}) = \frac{1}{N} \sum_{i=1}^N w_{\mathbr{\theta}'/\mathbr{\theta}}(\tau_i) R(\tau_i) -  \frac{\lambda}{\sqrt{\widehat{\ess}\left(p(\cdot|\mathbr{\theta}') \| p(\cdot|\mathbr{\theta})\right)}},
\end{equation*}
\begin{equation*}
	\widetilde{\mathcal{L}}_{\lambda}^{\mathrm{P-POIS}}(\mathbr{\rho}'/\mathbr{\rho}) = \frac{1}{N} \sum_{i=1}^N w_{\mathbr{\rho}'/\mathbr{\rho}}(\mathbr{\theta}_i) R(\tau_i) -  \frac{\lambda}{\sqrt{\widehat{\ess}\left(\nu_{\mathbr{\rho}'} \| \nu_{\mathbr{\rho}} \right) }}.
\end{equation*}
Moreover, in all the experiments, we use the empirical maximum reward in place of the true $R_{\max}$.

\subsection{Practical P-POIS for Deep Neural Policies (N-POIS)}
\label{apx:implPGPE}
As mentioned in Section \ref{sec:dnp}, P-POIS applied to deep neural policies suffers from a curse of dimensionality due to the high number of (scalar) parameters (which are $\sim 10^3$ for the network used in the experiments). The corresponding hyperpolicy is a multi-variate Gaussian (diagonal covariance) with a very high dimensionality. As a result, the \Renyi divergence, used as a penalty, is extremely sensitive even to small perturbations, causing an overly-conservative behavior.
First, we give up the exact \Renyi computation and use the practical surrogate objective function $\widetilde{\mathcal{L}}_{\lambda}^{\mathrm{P-POIS}}$ proposed in Appendix \ref{apx:psof}. This, however, is not enough. The importance weights, being the products of thousands of probability densities, can easily become zero, preventing any learning. Hence, we decide to group the policy parameters in smaller blocks, and independently learn the corresponding hyperparameters. In general, we can define a family of $M$ orthogonal policy-parameter subspaces $\{\Theta_m\leq\Theta\}_{m=1}^M$, where $V\leq W$ reads \quotes{$V$ is a subspace of $W$}. For each $\Theta_m$, we consider a multi-variate diagonal-covariance Gaussian with $\Theta_m$ as support, obtaining a corresponding hyperparameter subspace $\mathcal{P}_m\leq\mathcal{P}$. Then, for each $\mathcal{P}_m$, we compute a separate surrogate objective (where we employ self-normalized importance weights):
\begin{equation*}
\widetilde{\mathcal{L}}_{\lambda}^{\mathrm{N-POIS}}(\mathbr{\rho}_m'/\mathbr{\rho}_m) = \frac{1}{N} \sum_{i=1}^N \widetilde{w}_{\mathbr{\rho}_m'/\mathbr{\rho}_m}(\mathbr{\theta}^i_m) R(\tau_i) -  \frac{\lambda}{\sqrt{\widehat{\ess}\left(\nu_{\mathbr{\rho}_m'} \| \nu_{\mathbr{\rho}_m} \right) }},
\end{equation*}
where $\vrho_m, \vrho_m'\in\mathcal{P}_m, \vtheta_m\in\Theta_m$. 
Each objective is independently optimized via natural gradient ascent, where the step size is found via a line search as usual. It remains to define a meaningful grouping for the policy parameters, \ie for the weights of the deep neural policy. We choose to group them by network unit, or neuron (counting output units but not input units). More precisely, let denote a network unit as a function:
\[
	U_i(\mathbr{x}\vert\vtheta_m) = g(\mathbr{x}^T\vtheta_m),
\]
where $\mathbr{x}$ is the vector of the inputs to the unit (including a $1$ that multiplies the bias parameter) and $g(\cdot)$ is an activation function. To each unit $U_m$ we associate a block $\Theta_m$ such that $\vtheta_m\in\Theta_m$. In more connectivist-friendly terms, we group connections by the neuron they go into. For the network we used in the experiments, this reduces the order of the multivariate Gaussian hyperpolicies from $\sim10^3$ to $\sim10^2$. We call this practical variant of our algorithm Neuron-Based POIS (N-POIS). Although some design choices seem rather arbitrary, and independently optimizing hyperparameter blocks clearly neglects some potentially meaningful interactions, the practical results of N-POIS are promising, as reported in Section \ref{sec:dnp}.
Figure~\ref{fig:ablation} is an ablation study showing the performance of P-POIS variants on Cartpole. Only using both the tricks discussed in this section, we are able to solve the task (this experiment is on 50 iterations only).
\begin{figure}[t]
	\centering
	\includegraphics[]{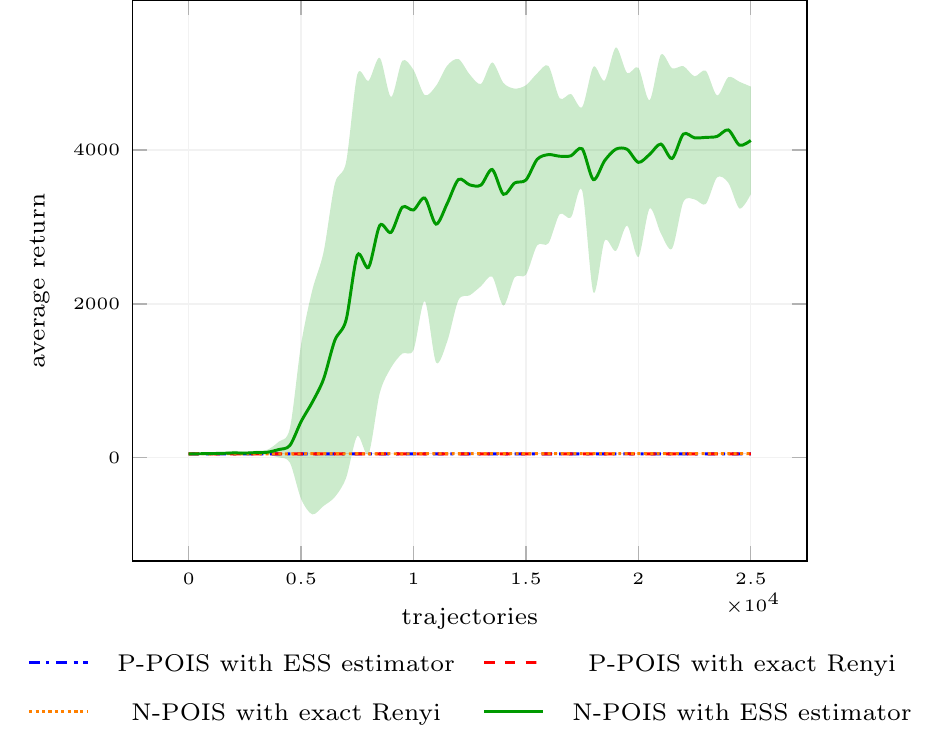}
	\caption{Ablation study for N-POIS (5 runs, 95\% c.i.).}
	\label{fig:ablation}
\end{figure}

\section{Experiments Details}
\label{apx:exp}
In this Appendix, we report the hyperparameter values used in the experimental evaluation and some additional plots and experiments. We adopted different criteria to decide the batch
size: for linear policies at each iteration 100 episodes are collected regardless of their length, whereas for deep neural policies, in order 
to be fully comparable with~\cite{duan2016benchmarking}, 50000 timesteps are collected at each iteration regardless of the resulting number of episodes (the last episode
is cut so that the number of timesteps sums up exactly to 50000). Clearly, this difference is relevant only for episodic tasks.
%

\subsection{Linear policies}
In the following we report the hyperparameters shared by all tasks and algorithms for the 
experiments with linear policies:
\begin{itemize}
	\item Policy architecture: Normal distribution $\mathcal{N}(u_{\mathbr{M}}(\mathbr{s}), e^{2\mathbr{\Omega}})$, where the mean $u_{\mathbr{M}}(\mathbr{s}) = \mathbr{M} \mathbr{s}$ is a linear function in the state variables with no bias, and the variance is state-independent and parametrized as $e^{2\mathbr{\Omega}}$, with diagonal $\mathbr{\Omega}$.
	\item Number of runs: 20 (95\% c.i.)
	\item seeds: \underline{10}, \underline{109}, \underline{904}, \underline{160}, \underline{570}, 662, 963, 100, 746, 236, 247, 689, 153,
      947, 307, 42, 950, 315, 545, 178
     \item Policy initialization: mean parameters sampled from $\mathcal{N}(0,0.01^2)$, variance initialized to 1
     \item Task horizon: 500
     \item Number of iterations: 500
     \item Maximum number of line search attempts (POIS only): 30
     \item Maximum number of offline iterations (POIS only): 10
     \item Episodes per iteration: 100
	 \item Importance weight estimator (POIS only): IS for A-POIS, SN for P-POIS
	 \item Natural gradient (POIS only): No for A-POIS, Yes for P-POIS
\end{itemize}
Table~\ref{tab:hypLin} reports the hyperparameters that have been tuned specifically for each task  selecting the
best combination based on the runs corresponding to the \underline{first 5 seeds}.

\begin{table}[!h]
  \caption{Task-specific hyperparameters for the experiments with linear policy. $\delta$ is the significance level for POIS while $\delta$ is the step-size for TRPO and PPO. In \textbf{bold}, the best hyperparameters found.}
  \label{tab:hypLin} 
  \centering
  \begin{tabular}{lcc}
    \toprule
    Environment     & A-POIS ($\delta$)     & P-POIS  ($\delta$) \\
    \midrule
    Cart-Pole Balancing & 0.1, 0.2, 0.3, \textbf{0.4}, 0.5 & 0.1, 0.2, 0.3, \textbf{0.4}, 0.5, 0.6, 0.7, 0.8, 0.9 1  \\
    Inverted Pendulum & 0.8, \textbf{0.9}, 0.99, 1 & 0.1, 0.2, 0.3, 0.4, 0.5, 0.6, 0.7, \textbf{0.8}, 0.9 1 \\
    Mountain Car & 0.8, \textbf{0.9}, 0.99, 1 & 0.1, 0.2, 0.3, 0.4, 0.5, 0.6, 0.7, 0.8, 0.9, \textbf{1}  \\
    Acrobot & 0.1, 0.3, 0.5, \textbf{0.7}, 0.9 & 0.1,\textbf{ 0.2}, 0.3, 0.4, 0.5, 0.6, 0.7, 0.8, 0.9 1  \\
    Double Inverted Pendulum & \textbf{0.1}, 0.2, 0.3, 0.4, 0.5& \textbf{0.1}, 0.2, 0.3, 0.4, 0.5, 0.6, 0.7, 0.8, 0.9 1  \\
    \bottomrule
  \end{tabular}
  \vspace{0.25cm}
  \vfill
  \begin{tabular}{lcc}
    \toprule
    Environment  & TRPO ($\delta$) & PPO ($\delta$) \\
    \midrule
    Cart-Pole Balancing & 0.001, 0.01, \textbf{0.1}, 1  & 0.001, \textbf{0.01}, 0.1 , 1 \\
    Inverted Pendulum & 0.001, \textbf{0.01}, 0.1, 1 & 0.001, \textbf{0.01}, 0.1, 1\\
    Mountain Car & 0.001, \textbf{0.01}, 0.1, 1 & 0.001, 0.01, 0.1, \textbf{1} \\
    Acrobot &  0.001, 0.01, 0.1, \textbf{1} &  0.001, 0.01, 0.1, \textbf{1}  \\
    Double Inverted Pendulum & 0.001, 0.01, \textbf{0.1}, 1   & 0.001, 0.01, 0.1, \textbf{1} \\
    \bottomrule
  \end{tabular}
\end{table}

\subsection{Deep neural policies}
In the following we report the hyperparameters shared by all tasks and algorithms for the 
experiments with deep neural policies:
\begin{itemize}
	\item Policy architecture: Normal distribution $\mathcal{N}(u_{\mathbr{M}}(\mathbr{s}), e^{2\mathbr{\Omega}})$, where the mean $u_{\mathbr{M}}(\mathbr{s})$ is a 3-layers MLP (100, 50, 25) with bias (activation functions: tanh for hidden-layers, linear for output layer), the variance is state-independent and parametrized as $e^{2\mathbr{\Omega}}$ with diagonal $\mathbr{\Omega}$.
	\item Number of runs: 5 (95\% c.i.)
	\item seeds: \underline{10}, \underline{109}, \underline{904}, \underline{160}, \underline{570}
     \item Policy initialization: uniform Xavier initialization~\cite{glorot2010understanding}, variance initialized to 1
     \item Task horizon: 500
     \item Number of iterations: 500
     \item Maximum number of line search attempts (POIS only): 30
     \item Maximum number of offline iterations (POIS only): 20
     \item Timesteps per iteration: 50000
	 \item Importance weight estimator (POIS only): IS for A-POIS, SN for P-POIS
	 \item Natural gradient (POIS only): No for A-POIS, Yes for P-POIS
\end{itemize}
Table~\ref{tab:hypNN} reports the hyperparameters that have been tuned specifically for each task  selecting the best combination based on the runs corresponding to the \underline{5 seeds}.

\begin{table}[!h]
  \caption{Task-specific hyperparameters for the experiments with deep neural policies. $\delta$ is the significance level for POIS. In \textbf{bold}, the best hyperparameters found.}
  \label{tab:hypNN}
  \centering
  \begin{tabular}{lcc}
     \toprule
    Environment     & A-POIS ($\delta$)     & P-POIS  ($\delta$)\\
    \midrule
    Cart-Pole Balancing &0.9, \textbf{0.99}, 0.999 & 0.4, 0.5, \textbf{0.6}, 0.7, 0.8 \\
    Mountain Car &  0.9, \textbf{0.99}, 0.999&    0.1, 0.2, \textbf{0.3}, 0.4, 0.5, 0.6, 0.7, 0.8\\
    Double Inverted Pendulum &  0.9, \textbf{0.99}, 0.999 &   0.4, 0.5, 0.6, 0.7, \textbf{0.8}\\
    Swimmer & 0.9, \textbf{0.99}, 0.999 &  0.4, 0.5, \textbf{0.6}, 0.7, 0.8 \\
    \bottomrule
  \end{tabular}
\end{table}

\subsection{Full experimental results}
In this section, we report the complete set of results we obtained by testing the two versions of POIS. 

In Figure~\ref{fig:plotDelta2} we report additional plots \wrt Figure~\ref{fig:plotDelta} for A-POIS when changing the $\delta$ parameter in the Cartpole environment. It is worth noting that the value of $\delta$ has also an effect on the speed with which the variance of the policy approaches zero. Indeed, smaller policy variances induce a larger \Renyi divergence and thus with a higher penalization (small $\delta$) reducing the policy variance is discouraged.
Moreover, we can see the values of the bound before and after the optimization. Clearly, the higher the value of $\delta$, the higher the value of the bound after the optimization process, as the penalization term is weaker. It is interesting to notice that when $\delta=1$ the bound after the optimization reaches values that are impossible to reach for any policy and this is a consequence of the high uncertainty in the importance sampling estimator.

\begin{figure}[h!]
\centering
\includegraphics[scale=1]{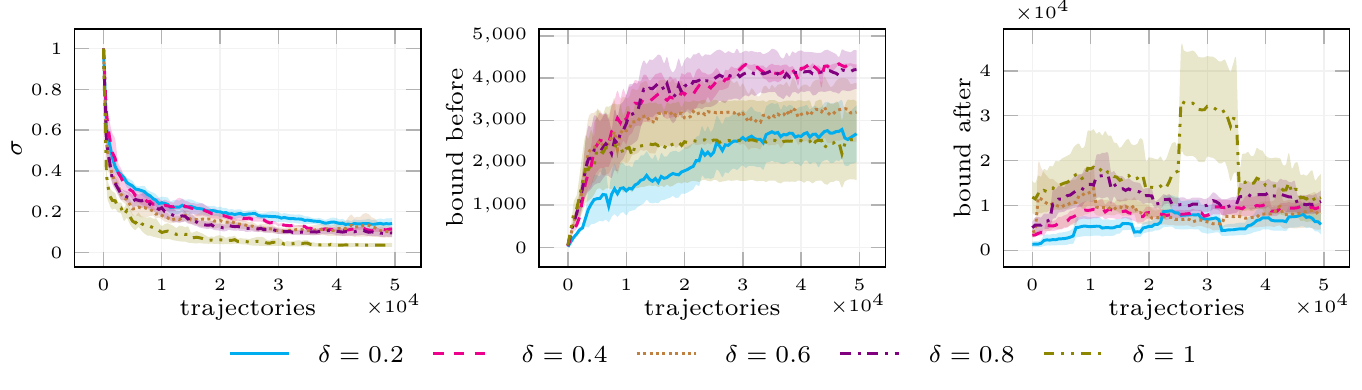}
\caption{Standard Deviation of the policy ($\sigma$), value of the bound before and after the optimization as a function of the number of trajectories for A-POIS in the Cartpole environment for different values of $\delta$ (5 runs, 95\% c.i.).}
\label{fig:plotDelta2}
\end{figure}

We report the comparative table taken from~\cite{duan2016benchmarking} containing all the benchmarked algorithms and the two versions of POIS (Table~\ref{tab:benchmarkapx}).

\setlength{\tabcolsep}{4pt}
\begin{table}[H]
  \caption{Cumulative return compared with~\cite{duan2016benchmarking} on deep neural policies (5 runs, 95\% c.i.). In \textbf{bold}, the performances that are not statistically significantly different from the best algorithm in each task.}
  \label{tab:benchmarkapx}
  \centering
  \footnotesize
  \begin{tabular}{lccccc}
    \toprule
        & Cart-Pole     &   & Double Inverted  &   \\
    	Algorithm		& Balancing & Mountain Car & Pendulum & Swimmer \\
    \midrule
    Random & $77.1 \pm 0.0  $  & $-415.4\pm 0.0$ & $149.7 \pm 0.1$ & $-1.7 \pm 0.1$    \\
    REINFORCE & $ 4693.7 \pm 14.0 $  & $-67.1 \pm 1.0 $ & $4116.5 \pm 65.2 $ & $92.3 \pm 0.1$     \\
    TNPG & $ \mathbf{3986.4 \pm 748.9} $  & $\mathbf{ -66.5 \pm 4.5 }$ & $\mathbf{4455.4 \pm 37.6}$ & $\mathbf{96.0 \pm 0.2}$      \\
    RWR & $ \mathbf{4861.5 \pm 12.3} $  & $ -79.4 \pm 1.1$ & $3614.8 \pm 368.1$ & $60.7 \pm 5.5$      \\
    REPS & $ 565.6 \pm 137.6 $  & $ -275.6 \pm 166.3 $ & $446.7\pm 114.8$ & $3.8 \pm 3.3$      \\
    TRPO & $ \mathbf{4869.8 \pm 37.6} $  & $\mathbf{ -61.7 \pm 0.9 }$ & $\mathbf{4412.4 \pm 50.4}$ & $\mathbf{96.0 \pm 0.2}$    \\
    DDPG & $ 4634.4 \pm 87.6 $  & $ -288.4 \pm 170.3$ & $2863.4 \pm 154.0$ & $85.8 \pm 1.8$      \\
    \rowcolor{blue!20}
    A-POIS & $ \mathbf{4842.8 \pm 13.0} $  & $ -63.7 \pm 0.5$ & $\mathbf{4232.1 \pm 189.5 }$ & $88.7 \pm 0.55$     \\
     \hdashline[3pt/2pt]
    CEM & $ 4815.4 \pm 4.8 $  & $ -66.0 \pm 2.4 $ & $2566.2 \pm 178.9$ & $68.8 \pm 2.4$     \\
    CMA-ES & $ 2440.4 \pm 568.3 $  & $ -85.0 \pm 7.7$ & $1576.1 \pm 51.3$ & $64.9 \pm 1.4$       \\
    \rowcolor{red!20}
    P-POIS & $ 4428.1 \pm 138.6 $  & $ -78.9 \pm 2.5 $ & $3161.4 \pm 959.2$  & $76.8 \pm 1.6$       \\
    \bottomrule
  \end{tabular}
\end{table}

In the following (Figure~\ref{fig:plotNN}) we show the learning curves of POIS in its two versions for the experiments with deep neural policies.


\begin{figure}[h!]
\centering
\footnotesize
\begin{subfigure}[l]{.5\textwidth}
  \centering
  \includegraphics[scale=1]{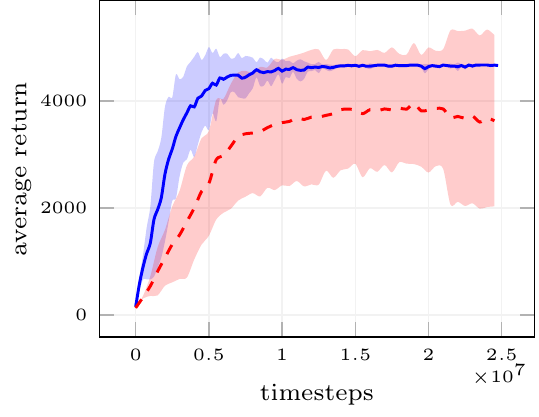}
  \caption{Inverted Double Pendulum}
\end{subfigure}%
\hfill
\begin{subfigure}[l]{.5\textwidth}
  \centering
   \includegraphics[scale=1]{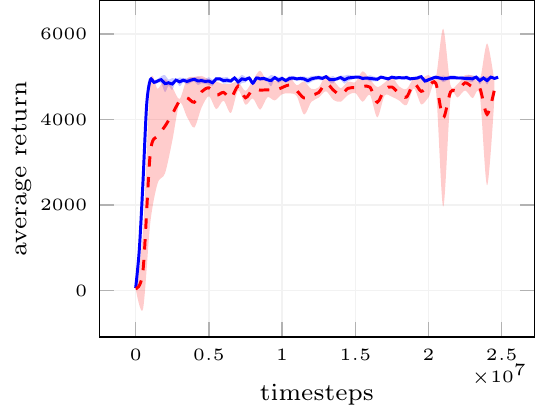}
   \caption{Cartpole}
\end{subfigure}%
\hfill
\begin{subfigure}[l]{.5\textwidth}
  \centering
   \includegraphics[scale=1]{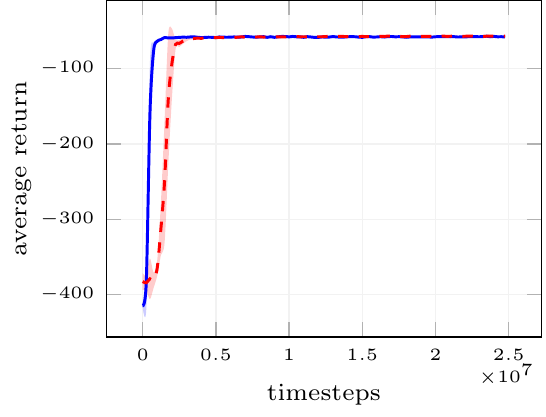}
  \caption{Mountain Car}
\end{subfigure}%
\hfill
\begin{subfigure}[l]{.5\textwidth}
  \centering
   \includegraphics[scale=1]{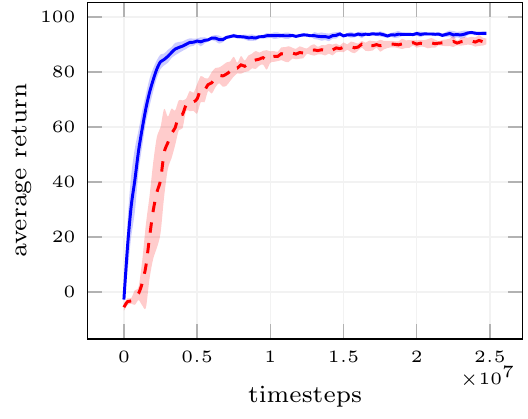}
   \caption{Swimmer}
\end{subfigure}%
\hfill
\begin{subfigure}[l]{\textwidth}
  \centering
   \includegraphics[scale=1]{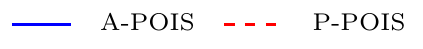}
\end{subfigure}%
\caption{Average return as a function of the number of trajectories for A-POIS, P-POIS with deep neural policies (5 runs, 95\% c.i.).}
\label{fig:plotNN}
\end{figure}

\end{document}